\documentclass{article}

\usepackage{microtype}
\usepackage[dvipdfmx]{graphicx}
\usepackage{subfigure}
\usepackage{booktabs} 

\usepackage{hyperref}

\usepackage[accepted]{icml2018}

\usepackage{natbib}

\renewcommand{\algorithmicrequire}{\textbf{Input:}}
\renewcommand{\algorithmicensure}{\textbf{Output:}}

\usepackage{amsthm}
\newtheorem{theorem}{Theorem}
\newtheorem{lemma}{Lemma}

\newtheorem{definition}{Definition}

\usepackage{amsmath}
\usepackage{amssymb}
\usepackage{bm}
\usepackage{bbm}
\usepackage{mathrsfs}
\usepackage{url}
\usepackage[font=small,skip=4pt]{caption}
\usepackage{pifont}
\usepackage{tabularx}

\usepackage{breakurl}
\newcommand{\vs}[1]{\vspace{#1}}
\newcommand{\noallowdisplaybreaks}{\interdisplaylinepenalty=10000}

\newcommand{\sign}{\mathop{\rm sign}}

\newcommand{\argmin}{\mathop{\rm arg~min}\limits}
\newcommand{\expect}{\mathop{\mathbb{E}}}

\newcommand{\s}{\mathrm{S}}
\renewcommand{\d}{\mathrm{D}}
\renewcommand{\u}{\mathrm{U}}
\newcommand{\cmark}{\ding{51}}
\newcommand{\xmark}{\ding{55}}

\makeatletter
\def\maketag@@@#1{\hbox{\m@th\normalsize\normalfont#1}}
\let\oldalign\align
\def\align{\def\f@size{10}\check@mathfonts\oldalign}
\makeatother
\renewcommand{\hat}{\widehat}
\renewcommand{\tilde}{\widetilde}
\newcommand{\reducespace}{\vs{-10pt}}

\icmltitlerunning{Classification from Pairwise Similarity and Unlabeled Data}

\begin{document}

\allowdisplaybreaks

\twocolumn[
\icmltitle{Classification from Pairwise Similarity and Unlabeled Data}




\begin{icmlauthorlist}
\icmlauthor{Han Bao}{u-tokyo,aip}
\icmlauthor{Gang Niu}{aip}
\icmlauthor{Masashi Sugiyama}{aip,u-tokyo}
\end{icmlauthorlist}

\icmlaffiliation{u-tokyo}{The University of Tokyo, Japan}
\icmlaffiliation{aip}{RIKEN, Japan}

\icmlcorrespondingauthor{Han Bao}{tsutsumi@ms.k.u-tokyo.ac.jp}

\icmlkeywords{binary classification, weakly-supervised classification, semi-supervised learning, semi-supervised clustering, metric learning, similarity, PU classification, class-prior estimation}

\vskip 0.3in
]



\printAffiliationsAndNotice{}  

\begin{abstract}
Supervised learning needs a huge amount of labeled data,
which can be a big bottleneck under the situation where there is a privacy concern or labeling cost is high.
To overcome this problem, we propose a new weakly-supervised learning setting
where only \emph{similar (S)} data pairs (two examples belong to the same class) and \emph{unlabeled (U)} data points are needed
instead of fully labeled data, which is called \emph{SU classification}.
We show that an unbiased estimator of the classification risk can be obtained only from SU data,
and the estimation error of its empirical risk minimizer achieves the optimal parametric convergence rate.
Finally, we demonstrate the effectiveness of the proposed method through experiments.
\end{abstract}

\section{Introduction}
In supervised classification, we need a vast amount of labeled data in the training phase.
However, in many real-world problems, it is time-consuming and laborious to label a huge amount of unlabeled data.
To deal with this problem, \emph{weakly-supervised classification}~\cite{Zhou:2018} has been explored in various setups,
including semi-supervised classification~\cite{Chapelle:2005,Belkin:2006,Chapelle:2010,Miyato:2016,Laine:2017,Sakai:2017a,Tarvainen:2017,Luo:2018},
multiple instance classification~\cite{Li:2015,Miech:2017,Bao:2018},
and positive-unlabeled (PU) classification~\cite{Elkan:2008,Plessis:2014,Plessis:2015,Niu:2016,Kiryo:2017}.

Another line of research from the clustering viewpoint is \emph{semi-supervised clustering},
where pairwise similarity and dissimilarity data (a.k.a.~must-link and cannot-link constraints) are utilized to guide unsupervised clustering to a desired solution.
The common approaches are
(i) constrained clustering~\cite{Wagstaff:2001,Basu:2002,Basu:2004,Li:2009:ICCV}, which utilize pairwise links as constraints on clustering.
(ii) metric learning~\cite{Xing:2002,Bilenko:2004,Weinberger:2005,Davis:2007,Li:2008,Niu:2012}, which perform ($k$-means) clustering on learned metrics
(iii) matrix completion~\cite{Yi:2013,Chiang:2015}, which recover unknown entries in a similarity matrix.

\begin{table}[t]
  \centering
  \caption{Explanations of classification and clustering.}
  \label{tab:problem-definition}
  \begin{minipage}{\columnwidth}
    \begin{tabularx}{\columnwidth}{c|X} \hline
      Problem & Explanation \\ \hline \hline
      Classification &
      The goal is to minimize the true risk (given the zero-one loss) of an inductive classifier.
      To this end, an empirical risk (given a surrogate loss) on the training data is minimized for training the classifier.
      The training and testing phases can be clearly distinguished.
      Classification requires the existence of the underlying joint density.
      \\ \hline
      Clustering &
      The goal is to partition the data at hand into clusters.
      To this end, density-/margin-/information-based measures are optimized for implementing the low-density separation based on the cluster assumption.
      Most of the clustering methods are designed for in-sample inference\footnote{
        Discriminative clustering methods are designed for out-of-sample inference,
        such as maximum margin clustering~\cite{Xu:2005} and information maximization clustering~\cite{Krause:2010,Sugiyama:2011}.
      }.
      Clustering does not need the underlying joint density.
      \\ \hline
    \end{tabularx}
  \end{minipage}
  \reducespace
\end{table}

Semi-supervised clustering and weakly-supervised classification
are similar in that they do not use fully-supervised data.
However, they are different from the learning theoretic viewpoint---weakly-supervised classification methods are justified as supervised learning methods,
while semi-supervised clustering methods are still evaluated as unsupervised learning
(see Table~\ref{tab:problem-definition}).
Indeed, weakly-supervised learning methods based on empirical risk minimization~\cite{Plessis:2014,Plessis:2015,Niu:2016,Sakai:2017a} were shown that
their estimation errors achieve the optimal parametric convergence rate,
while such generalization guarantee is not available for semi-supervised clustering methods.

The goal of this paper is to propose a novel weakly-supervised learning method called \emph{SU classification},
where only \emph{similar (S)} data pairs (two examples belong to the same class) and \emph{unlabeled (U)} data points are employed,
in order to bridge these two different paradigms.
In SU classification, the information available for training a classifier is similar to semi-supervised clustering.
However, our proposed method gives an \emph{inductive model},
which learns decision functions from training data and can be applied for out-of-sample prediction (i.e., prediction of unseen test data).
Furthermore, the proposed method can not only separate two classes but also \emph{identify which class is positive} (class identification) under certain conditions.

SU classification is particularly useful to predict people's \emph{sensitive matters} such as religion, politics, and opinions on racial issues---people often hesitate to give explicit answers to these matters,
instead indirect questions might be easier to answer: ``Which person do you have the same belief as?"\footnote{
  This questioning can be regarded as one type of randomized response (indirect questioning) techniques~\cite{Warner:1965,Fisher:1993},
  which is a survey method to avoid social desirability bias.
}

For this SU classification problem, our contributions in this paper are three-fold:
\vs{-3pt}
\begin{enumerate}
  \item
  We propose an empirical risk minimization method for SU classification (Section~\ref{sec:su-learning}).
  This enables us to obtain an inductive classifier.
  Under certain loss conditions together with the linear-in-parameter model, its objective function becomes even convex in the parameters.

  \item
  We theoretically establish an estimation error bound for our SU classification method (Section~\ref{sec:theory}), showing that the proposed method achieves the optimal parametric convergence rate.

  \item
  We experimentally demonstrate the practical usefulness of the proposed SU classification method (Section~\ref{sec:experiments}).
\end{enumerate}
\vs{-3pt}

Related problem settings are summarized in Figure~\ref{fig:problem-settings}.

\begin{figure}[t]
  \centering
  \includegraphics[width=\columnwidth]{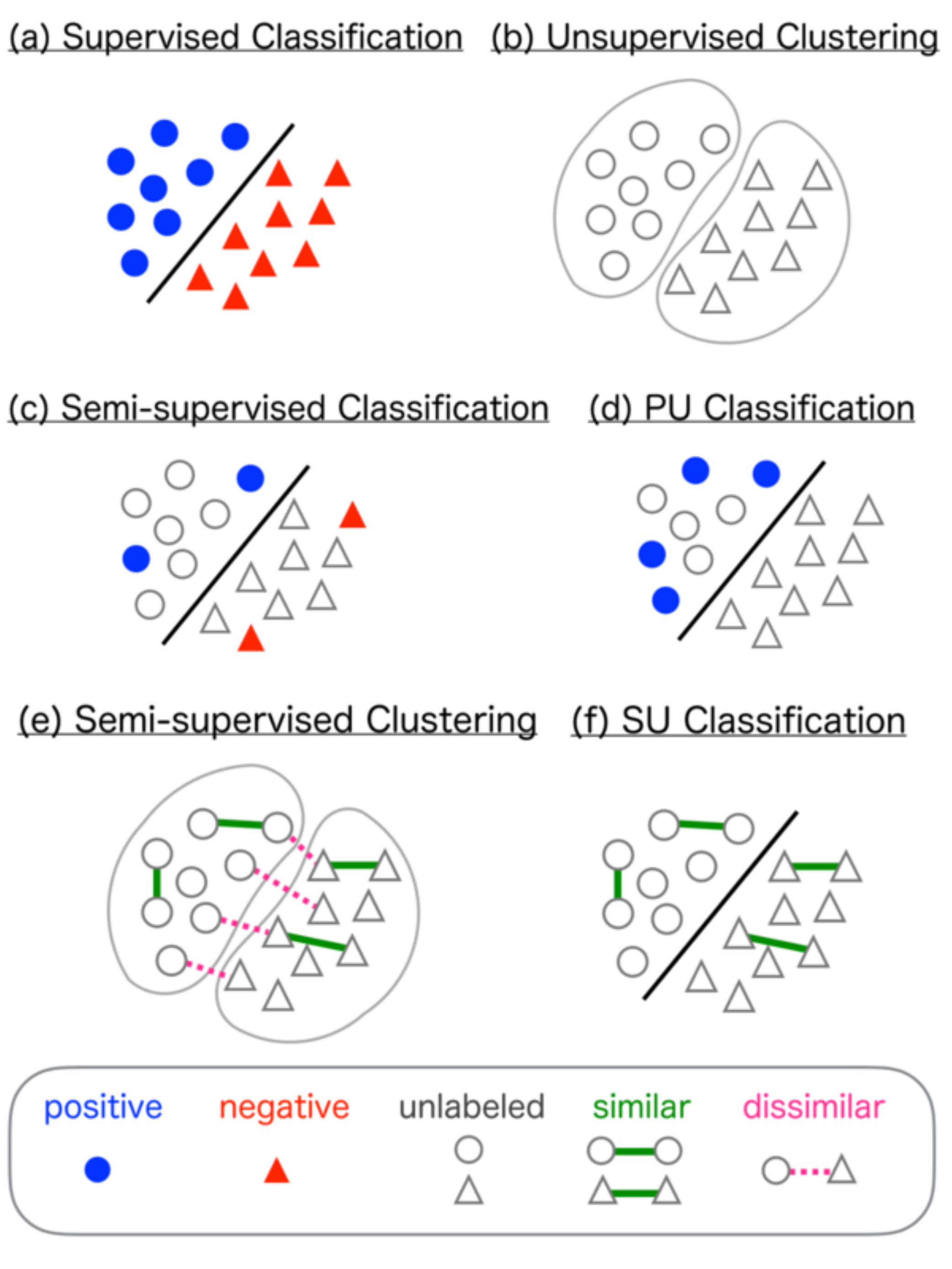}
  \caption{
    Illustrations of SU classification and other related problem settings.
  }
  \label{fig:problem-settings}
\end{figure}

\section{Classification from Pairwise Similarity and Unlabeled Data}
\label{sec:su-learning}
In this section, we propose a learning method to train a classifier 
from pairwise similarity and unlabeled data.

\subsection{Preliminaries}
\label{sec:preliminaries}
We formulate the standard binary classification problem briefly.
Let $\mathcal{X} \subset \mathbb{R}^d$ be a $d$-dimensional example space
and $\mathcal{Y} = \{+1, -1\}$ be a binary label space.
We assume that labeled data $(\bm{x},y) \in \mathcal{X} \times \mathcal{Y}$ is drawn
from the joint probability distribution with density $p(\bm{x},y)$.
The goal of binary classification is
to obtain a classifier $f: \mathcal{X} \rightarrow \mathbb{R}$
which minimizes the classification risk defined as
\begin{align}
  R(f) \triangleq \expect_{(X,Y) \sim p} \left[\ell(f(X), Y)\right]
  \label{eq:risk}
  ,
\end{align}
where $\mathbb{E}_{(X,Y) \sim p} [\cdot]$ denotes the expectation over the joint distribution $p(X,Y)$
and $\ell: \mathbb{R} \times \mathcal{Y} \rightarrow \mathbb{R}_{+}$ is a loss function.
The loss function $\ell(z, t)$ measures how well the true class label $t \in \mathcal{Y}$ is estimated by an output of a classifier $z \in \mathbb{R}$,
generally yielding a small/large value if $t$ is well/poorly estimated by $z$.

In standard supervised classification scenarios, we are given
positive and negative training data independently following $p(\bm{x},y)$.
Then, based on these training data, the classification risk \eqref{eq:risk}
is empirically approximated and the empirical risk minimizer is obtained.
However, in many real-world problems, collecting labeled training data is costly.
The goal of this paper is to train a binary classifier only from
pairwise similarity and unlabeled data, which are cheaper to collect
than fully labeled data.

\subsection{Pairwise Similarity and Unlabeled Data}

First, we discuss underlying distributions of similar data pairs and unlabeled data points,
in order to perform the empirical risk minimization.

\textbf{Pairwise Similarity:}
If $\bm{x}$ and $\bm{x}'$ belong to the same class,
they are said to be \emph{pairwise similar} (S).
We assume that similar data pairs are drawn following
\noallowdisplaybreaks
\begin{align}
  p_{\s}(\bm{x},\bm{x}')
  &= p(\bm{x}, \bm{x}'| y=y'=+1 \vee y=y'=-1) \nonumber \\
  &= \frac{\pi_+^2 p_+(\bm{x}) p_+(\bm{x}') + \pi_-^2 p_-(\bm{x}) p_-(\bm{x}')}{\pi_+^2 + \pi_-^2}
  \label{eq:pairwise-similarity-conditional}
  ,
\end{align}
\allowdisplaybreaks
where
$\pi_+ \triangleq p(y = +1)$ and $\pi_- \triangleq p(y = -1)$
are the class-prior probabilities satisfying $\pi_+ + \pi_- = 1$,
and
$p_+(\bm{x}) \triangleq p(\bm{x} | y = +1)$ and $p_-(\bm{x}) \triangleq p(\bm{x} | y = -1)$
are the class-conditional densities.
Eq.~\eqref{eq:pairwise-similarity-conditional} means that we draw two labeled data independently following $p(\bm{x},y)$,
and we accept/reject them if they belong to the same class/different classes.

\textbf{Unlabeled Data:}
We assume that unlabeled (U) data points are drawn following
the marginal density $p(\bm{x})$,
which can be decomposed into the sum of the class-conditional densities as
\vs{-3pt}
\begin{align}
  p(\bm{x})
  &= \pi_+p_+(\bm{x}) + \pi_-p_-(\bm{x})
  \label{eq:marginal}
  .
\end{align}
\vs{-5pt}

Our goal is to train a classifier only from SU data,
which we call \emph{SU classification}.
We assume that we have similar pairs $\mathcal{D}_\s$ and an unlabeled dataset $\mathcal{D}_\u$ as
\noallowdisplaybreaks
\begin{alignat*}{2}
  \mathcal{D}_\s &\triangleq \{(\bm{x}_{\s,i},\bm{x}_{\s,i}')\}_{i=1}^{n_\s} && \stackrel{\text{i.i.d.}}{\sim} p_\s(\bm{x},\bm{x}'), \\
  \mathcal{D}_\u &\triangleq \{\bm{x}_{\u,i}\}_{i=1}^{n_\u} && \stackrel{\text{i.i.d.}}{\sim} p(\bm{x})
  .
\end{alignat*}
\allowdisplaybreaks
We also use a notation $\tilde{\mathcal{D}}_\s \triangleq \{\tilde{\bm{x}}_{\s,i}\}_{i=1}^{2n_\s}$ to denote pointwise similar data obtained by ignoring pairwise relations in $\mathcal{D}_\s$.

\begin{lemma}
  \label{lem:pointwise-similar-conditional}
  $\tilde{\mathcal{D}}_\s = \{\tilde{\bm{x}}_{\s,i}\}_{i=1}^{2n_\s}$ are independently drawn following
  \begin{align}
    \tilde{p}_\s(\bm{x}) = \frac{\pi_+^2p_+(\bm{x}) + \pi_-^2p_-(\bm{x})}{\pi_\s}
    \label{eq:pointwise-similar-conditional}
    ,
  \end{align}
  where $\pi_\s \triangleq \pi_+^2 + \pi_-^2$.
\end{lemma}
A proof is given in Appendix~\ref{sec:proof-of-pointwise-similar-conditional}.

Lemma~\ref{lem:pointwise-similar-conditional} states that a similar data pair $(\bm{x}_\s,\bm{x}_\s')$ is essentially symmetric, and $\bm{x}_\s, \bm{x}_\s'$ can be regarded as being independently drawn following $\tilde{p}_\s$,
if we assume the pair $(\bm{x}_\s,\bm{x}_\s')$ is drawn following $p_\s$.
This perspective is important when we analyze the variance of the risk estimator (Section~\ref{sec:variance-minimization}), and estimate the class-prior (Section~\ref{sec:class-prior-estimation}).

\subsection{Risk Expression with SU Data}

Below, we attempt to express the classification risk \eqref{eq:risk}
only in terms of SU data.
Assume $\pi_+ \ne \frac{1}{2}$,
and let
$\tilde{\ell}(z)$,
$\mathcal{L}_{\s,\ell}(z)$
and $\mathcal{L}_{\u,\ell}(z)$ be
  \begin{align*}
    \tilde{\ell}(z) &\triangleq \ell(z,+1) - \ell(z,-1), \\
    \mathcal{L}_{\s,\ell}(z) &\triangleq \frac{1}{2\pi_+ - 1}\tilde{\ell}(z), \\
    \mathcal{L}_{\u,\ell}(z) &\triangleq - \frac{\pi_-}{2\pi_+ - 1}\ell(z,+1) + \frac{\pi_+}{2\pi_+ - 1}\ell(z,-1)
    .
  \end{align*}
Then we have the following theorem.

\noallowdisplaybreaks
\begin{theorem}
  \label{thm:su-risk}
  The classification risk \eqref{eq:risk} can be equivalently expressed as
  \begin{align}
    R_{\mathrm{SU},\ell}(f)
    &= \pi_\s\expect_{(X,X')\sim p_{\s}}\left[ \frac{\mathcal{L}_{\s,\ell}(f(X)) + \mathcal{L}_{\s,\ell}(f(X'))}{2} \right] \nonumber \\
    & \quad + \expect_{X \sim p}\left[ \mathcal{L}_{\u,\ell}(f(X)) \right]
    \nonumber
    .
  \end{align}
\end{theorem}
\allowdisplaybreaks
A proof is given in Appendix~\ref{sec:proof-of-su-estimator}.

\begin{figure*}[t]
  \begin{minipage}{0.24\hsize}
    \centering
    \includegraphics[width=\columnwidth]{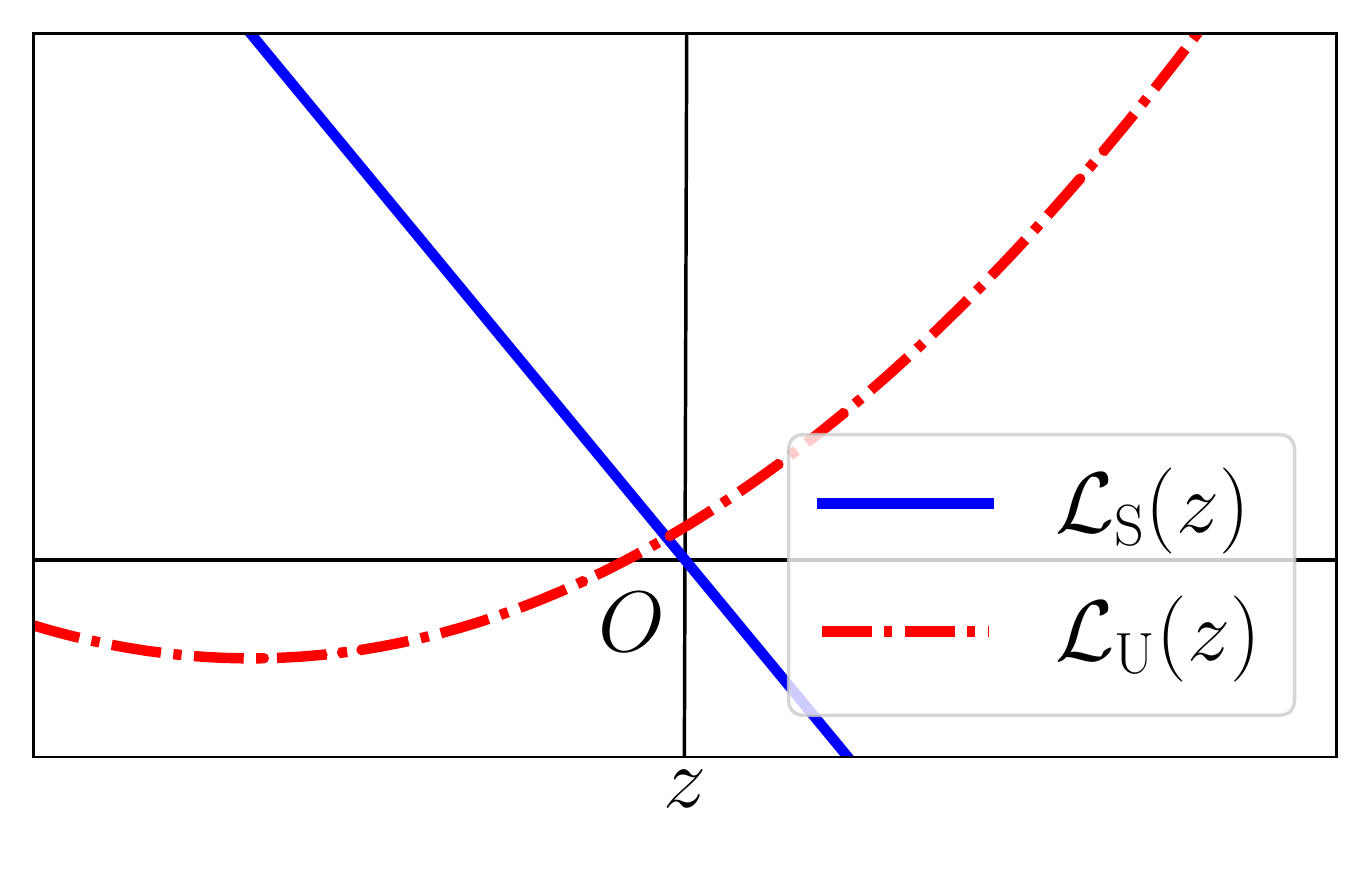}
    \caption*{(a) Squared Loss, $\pi_+=\frac{3}{4}$}
  \end{minipage}
  \begin{minipage}{0.24\hsize}
    \centering
    \includegraphics[width=\columnwidth]{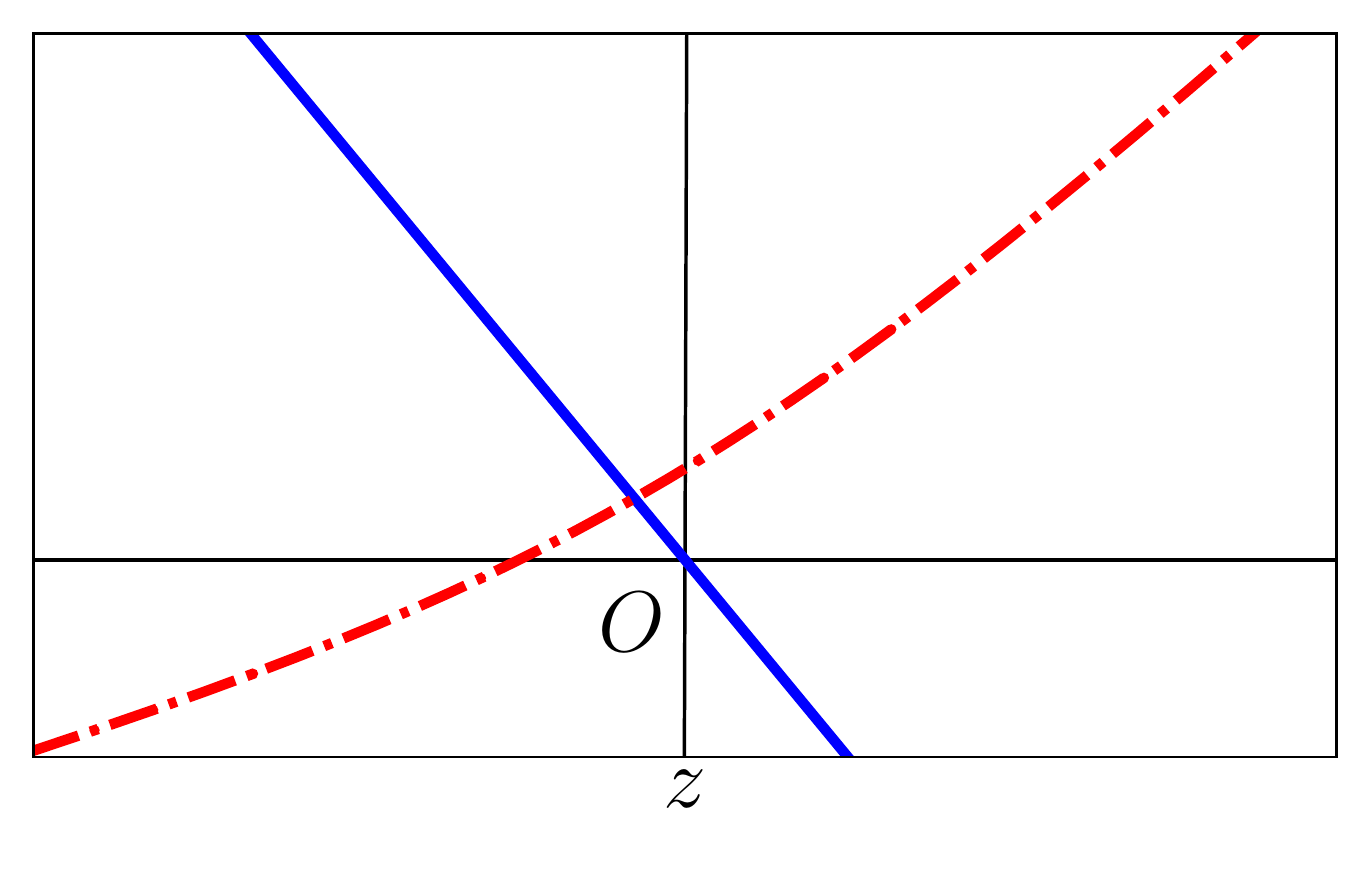}
    \caption*{(b) Logistic Loss, $\pi_+=\frac{3}{4}$}
  \end{minipage}
  \begin{minipage}{0.24\hsize}
    \centering
    \includegraphics[width=\columnwidth]{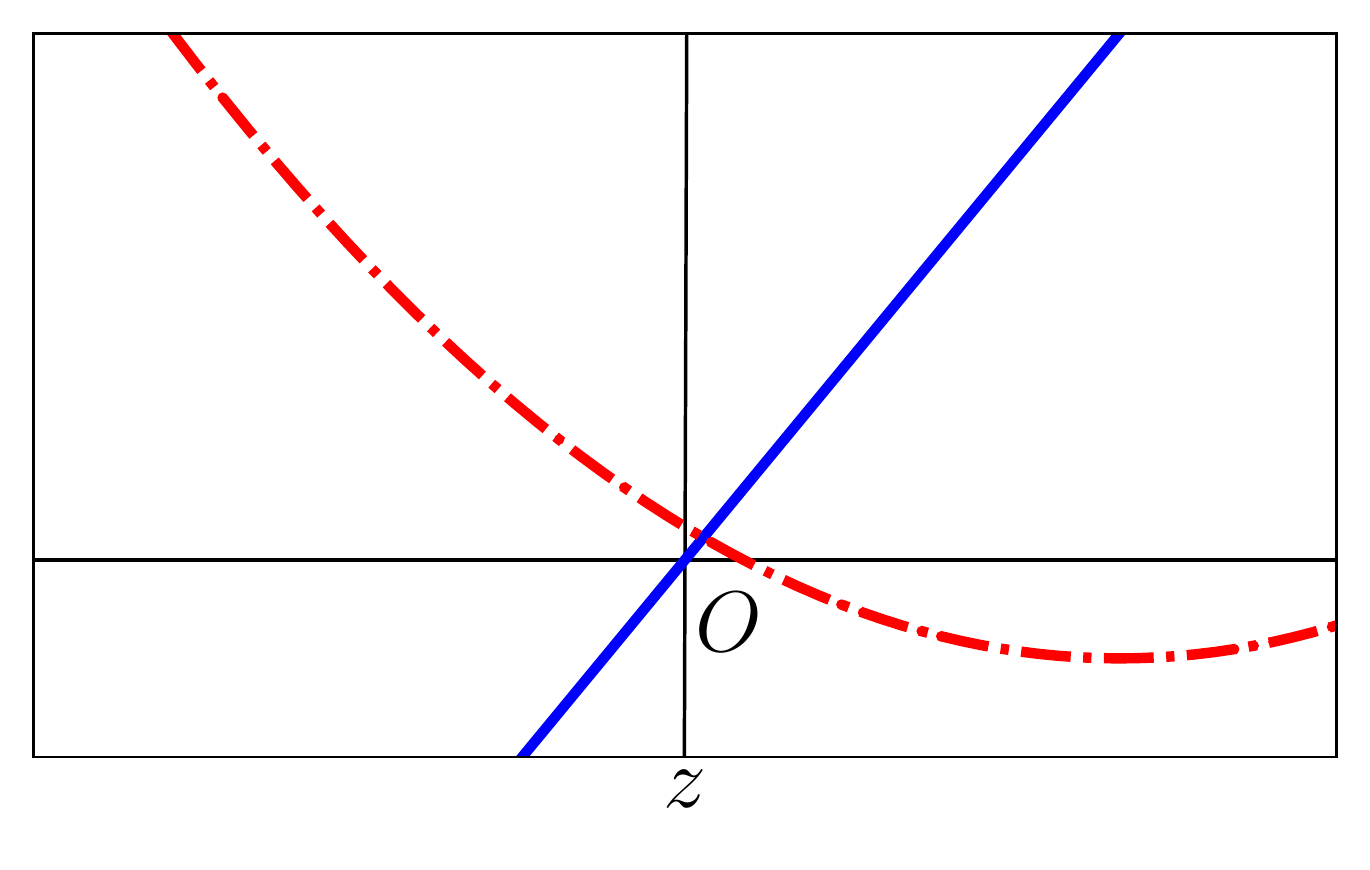}
    \caption*{(c) Squared Loss, $\pi_+=\frac{1}{4}$}
  \end{minipage}
  \begin{minipage}{0.24\hsize}
    \centering
    \includegraphics[width=\columnwidth]{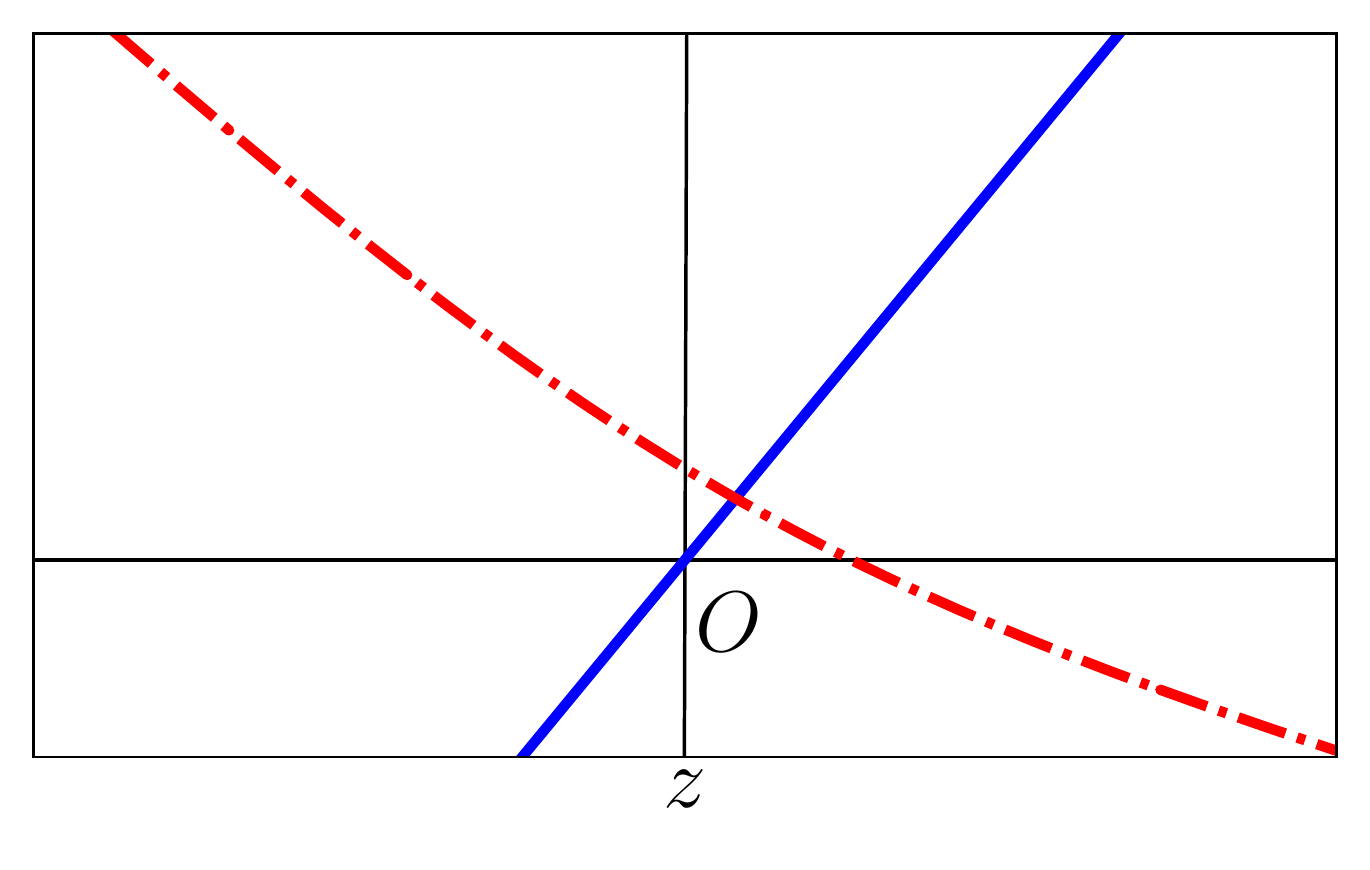}
    \caption*{(d) Logistic Loss, $\pi_+=\frac{1}{4}$}
  \end{minipage}
  \caption{
  $\mathcal{L}_{\s,\ell}$ and $\mathcal{L}_{\u,\ell}$ appearing in Eq.~\eqref{eq:su-estimator}
  are illustrated with different loss functions and class-priors.
  }
  \label{fig:su-loss-functions}
\end{figure*}

According to Theorem~\ref{thm:su-risk}, the following is a natural candidate for an unbiased estimator of the classification risk~\eqref{eq:risk}:
\noallowdisplaybreaks
\begin{align}
  \hat{R}&_{\mathrm{SU},\ell}(f) \nonumber
  \\
  &= \frac{\pi_\s}{n_\s} \sum_{i=1}^{n_\s} \frac{\mathcal{L}_{\s,\ell}(f(\bm{x}_{\s,i})) + \mathcal{L}_{\s,\ell}(f(\bm{x}_{\s,i}'))}{2} \nonumber \\
  & \qquad + \frac{1}{n_\u} \sum_{i=1}^{n_\u} \mathcal{L}_{\u,\ell}(f(\bm{x}_{\u,i})) \nonumber
  \\
  &= \frac{\pi_\s}{2n_\s} \sum_{i=1}^{2n_\s} \mathcal{L}_{\s,\ell}(f(\tilde{\bm{x}}_{\s,i}))
  + \frac{1}{n_\u} \sum_{i=1}^{n_\u} \mathcal{L}_{\u,\ell}(f(\bm{x}_{\u,i}))
  \label{eq:su-estimator}
  ,
\end{align}
\allowdisplaybreaks
where in the last line we use the decomposed version of similar pairs $\tilde{\mathcal{D}}_\s$ instead of $\mathcal{D}_\s$,
since the loss form is symmetric.

$\mathcal{L}_{\s,\ell}$ and $\mathcal{L}_{\u,\ell}$ are illustrated in Figure~\ref{fig:su-loss-functions}.

\subsection{Minimum-Variance Risk Estimator}
\label{sec:variance-minimization}

Eq.~\eqref{eq:su-estimator} is one of the candidates of an unbiased SU risk estimator.
Indeed, due to the symmetry of $(\bm{x},\bm{x}') \sim p_\s(\bm{x},\bm{x}')$, we have the following lemma.
\begin{lemma}
  \label{lem:pairing-fact}
  The first term of $R_{\mathrm{SU},\ell}(f)$, i.e.,
  \begin{align}
    \pi_\s \expect_{(X,X') \sim p_\s} \left[ \frac{\mathcal{L}_{\s,\ell}(f(X)) + \mathcal{L}_{\s,\ell}(f(X'))}{2} \right],
    \label{eq:s-risk}
  \end{align}
  can be equivalently expressed as
  \begin{align*}
    \pi_\s\expect_{(X,X') \sim p_\s} \left[ \alpha\mathcal{L}_{\s,\ell}(f(X)) + (1-\alpha)\mathcal{L}_{\s,\ell}(f(X')) \right],
  \end{align*}
  where $\alpha \in [0, 1]$ is an arbitrary weight.
\end{lemma}
A proof is given in Appendix~\ref{sec:proof-of-pairing-fact}.
By Lemma~\ref{lem:pairing-fact},
\begin{align}
  \frac{\pi_\s}{n_\s}\sum_{i=1}^{n_\s}\left\{\alpha\mathcal{L}_{\s,\ell}(f(\bm{x}_{\s,i})) + (1-\alpha)\mathcal{L}_{\s,\ell}(f(\bm{x}_{\s,i}')) \right\}
  \label{eq:another-s-risk}
\end{align}
is also an unbiased estimator of Eq.~\eqref{eq:s-risk}.
Then, a natural question arises: \textit{is the risk estimator~\eqref{eq:su-estimator} best among all $\alpha$?}
We answer this question by the following theorem.
\begin{theorem}
  \label{thm:minimum-variance-pairing}
  The estimator
  \begin{align}
    \frac{\pi_\s}{n_\s} \sum_{i=1}^{n_\s} \frac{\mathcal{L}_{\s,\ell}(f(\bm{x}_{\s,i})) + \mathcal{L}_{\s,\ell}(f(\bm{x}_{\s,i}'))}{2}
    \label{eq:minimum-variance-s-estimator}
  \end{align}
  has the minimum variance among estimators in the form Eq.~\eqref{eq:another-s-risk} with respect to $\alpha \in [0,1]$.
\end{theorem}
A proof is given in Appendix~\ref{sec:proof-of-minimum-variance}.

Thus, the variance minimality (with respect to $\alpha$ in Eq.~\eqref{eq:another-s-risk}) of the risk estimator~\eqref{eq:su-estimator} is guaranteed by Theorem~\ref{thm:minimum-variance-pairing}.
We use this risk estimator in the following sections.

\subsection{Practical Implementation}

Here, we investigate the objective function when the linear-in-parameter model $f(\bm{x}) = \bm{w}^\top \bm{\phi}(\bm{x}) + w_0$ is employed as a classifier,
where $\bm{w} \in \mathbb{R}^d$ and $w_0 \in \mathbb{R}$ are parameters and $\bm{\phi}: \mathbb{R}^d \rightarrow \mathbb{R}^b$ is basis functions.
In general, the bias parameter $w_0$ can be ignored~\footnote{
  Let $\tilde{\bm{\phi}}(\bm{x}) \triangleq [\bm{\phi}(\bm{x})^\top\; 1]^\top$ and $\tilde{\bm{w}} \triangleq [\bm{w}^\top\; w_0]^\top$
  then $\bm{w}^\top\bm{\phi}(\bm{x}) + w_0 = \tilde{\bm{w}}^\top \tilde{\bm{\phi}}(\bm{x})$.
}.
We formulate SU classification as the following empirical risk minimization problem using Eq.~\eqref{eq:su-estimator} together with the $\ell_2$ regularization:
\begin{align}
  \hat{\bm{w}} = \min_{\bm{w}} \hat{J}_\ell(\bm{w})
  \label{eq:optimization-problem}
  ,
\end{align}
\vs{-10pt}
where
\noallowdisplaybreaks
\begin{align}
  \hat{J}_{\ell}(\bm{w})
  &\triangleq \frac{\pi_\s}{2n_\s} \sum_{i=1}^{2n_\s} \mathcal{L}_{\s,\ell}(\bm{w}^\top \bm{\phi}(\tilde{\bm{x}}_{\s,i}))
  \nonumber \\
  & \phantom{(\pi_+} + \frac{1}{n_\u} \sum_{i=1}^{n_\u} \mathcal{L}_{\u,\ell}(\bm{w}^\top \bm{\phi}(\bm{x}_{\u,i}))
  + \frac{\lambda}{2}\|\bm{w}\|^2
  \label{eq:su-objective-function}
  ,
\end{align}
\allowdisplaybreaks
and $\lambda > 0$ is the regularization parameter.
We need the class-prior $\pi_+$ (included in $\pi_\s$) to solve this optimization problem.
We discuss how to estimate $\pi_+$ in Section~\ref{sec:class-prior-estimation}.

Next, we will investigate appropriate choices of the loss function $\ell$.
From now on, we focus on \emph{margin loss functions}~\cite{Mohri:2012}:
$\ell$ is said to be a margin loss function if there exists $\psi: \mathbb{R} \rightarrow \mathbb{R}_+$ such that $\ell(z,t) = \psi(tz)$.

In general, our objective function~\eqref{eq:su-objective-function} is non-convex even if a convex loss function is used for $\ell$~\footnote{
  In general, $\mathcal{L}_{\u,\ell}$ is non-convex because either $-\frac{\pi_-}{2\pi_+ - 1}\ell(\cdot,+1)$ or $\frac{\pi_+}{2\pi_+ - 1}\ell(\cdot,-1)$ is convex and the other is concave.
  $\mathcal{L}_{\s,\ell}$ is not always convex even if $\ell$ is convex, either.
}.
However, the next theorem, inspired by \citet{Natarajan:2013} and \citet{Plessis:2015}, states that a certain loss function will result in a convex objective function.
\begin{theorem}
  \label{thm:convex-su}
  If the loss function $\ell(z,t)$ is a convex margin loss,
  twice differentiable in $z$ almost everywhere (for every fixed $t \in \{\pm 1\}$),
  and satisfies the condition
  \vs{-3pt}
  \begin{align}
    \ell(z,+1) - \ell(z,-1) = -z \nonumber
    ,
  \end{align}
  \vs{-5pt}
  then $\hat{J}_\ell(\bm{w})$ is convex.
\end{theorem}
A proof of Theorem~\ref{thm:convex-su} is given in Appendix~\ref{sec:proof-of-convex-su}.

\begin{table}[t]
  \centering
  \caption{A selected list of margin loss functions satisfying the conditions in Theorem~\ref{thm:convex-su}.}
  \label{tab:loss-functions}
  \begin{tabular}{cc} \hline
    Loss name         & $\psi(m)$ \\ \hline
    Squared loss    & $\frac{1}{4}(m - 1)^2$ \\
    Logistic loss     & $\log(1 + \exp(-m))$ \\
    Double hinge loss & $\max(-m, \max(0, \frac{1}{2} - \frac{1}{2}m))$ \\
    \hline
  \end{tabular}
  \reducespace
\end{table}

Examples of margin loss functions satisfying the conditions in Theorem~\ref{thm:convex-su} are shown in Table~\ref{tab:loss-functions}
(also illustrated in Figure~\ref{fig:loss-function}).
Below, as special cases, we show the objective functions for the squared and the double-hinge losses.
The detailed derivations are given in Appendix~\ref{sec:optimization-problem}.

\begin{figure}[t]
  \centering
  \includegraphics[width=\columnwidth]{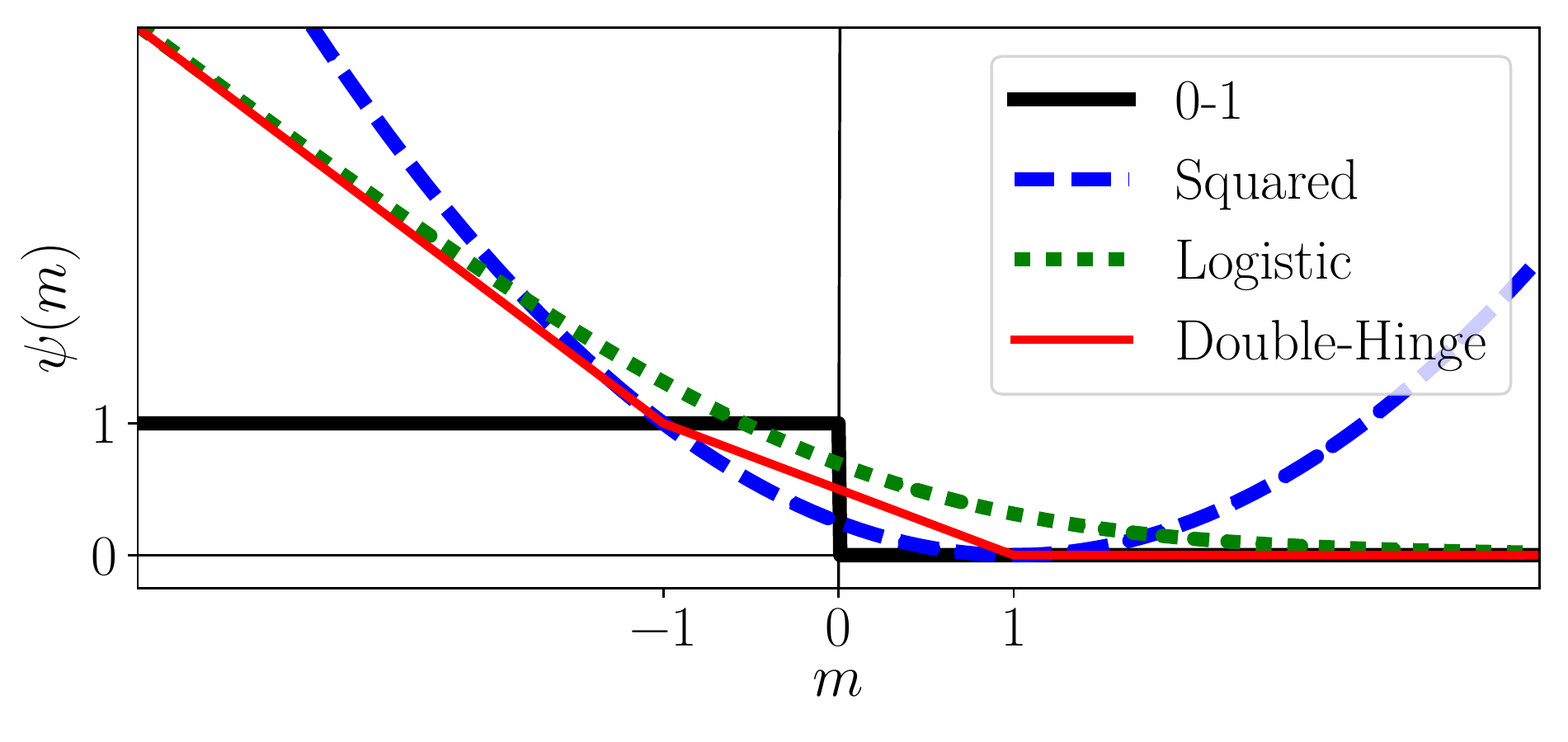}
  \caption{Comparison of loss functions.}
  \label{fig:loss-function}
\end{figure}

\textbf{Squared Loss: }
The squared loss is $\ell_\mathrm{SQ}(z,t) = \frac{1}{4}(tz - 1)^2$.
Substituting $\ell_\mathrm{SQ}$ into Eq.~\eqref{eq:su-objective-function}, the objective function is
\noallowdisplaybreaks
\begin{align*}
  \hat{J}_\mathrm{SQ}(\bm{w})
  &= \bm{w}^\top\left(\frac{1}{4n_\u}X_\u^\top X_\u + \frac{\lambda}{2}I\right)\bm{w} \\
  & \quad + \frac{1}{2\pi_+ - 1}\left(-\frac{\pi_\s}{2n_\s}\bm{1}^\top X_\s + \frac{1}{2n_\u}\bm{1}^\top X_\u\right)\bm{w}
  ,
\end{align*}
\allowdisplaybreaks
where $\bm{1}$ is the vector whose elements are all ones, $I$ is the identity matrix,
$X_\s \triangleq [\bm{\phi}(\tilde{\bm{x}}_{\s,1})\;\cdots\;\bm{\phi}(\tilde{\bm{x}}_{\s,2n_{\s}})]^\top$,
and $X_\u \triangleq [\bm{\phi}(\bm{x}_{\u,1})\;\cdots\;\bm{\phi}(\bm{x}_{\u,n_\u})]^\top$.
The minimizer of this objective function can be obtained analytically as
\noallowdisplaybreaks
\begin{align}
  \bm{w}
  &=
  \frac{n_\u}{2\pi_+ - 1} \nonumber \\
  &\cdot\left(X_\u^\top X_\u + 2\lambda n_\u I\right)^{-1}
  \left(\frac{\pi_\s}{n_\s}X_\s^\top\bm{1} - \frac{1}{n_\u}X_\u^\top\bm{1}\right)
  \nonumber
  .
\end{align}
\allowdisplaybreaks
Thus the optimization problem can be easily implemented and solved highly efficiently if the number of basis functions is not so large.

\textbf{Double-Hinge Loss:}
Since the hinge loss $\ell_\mathrm{H}(z,t)=\max(0,1-tz)$ does not satisfy the conditions in Theorem~\ref{thm:convex-su},
the double-hinge loss $\ell_\mathrm{DH}(z,t)=\max(-tz,\max(0,\frac{1}{2}-\frac{1}{2}tz))$ is proposed by \citet{Plessis:2015} as an alternative.
Substituting $\ell_\mathrm{DH}$ into Eq.~\eqref{eq:su-objective-function}, we can reformulate the optimization problem as follows:
\noallowdisplaybreaks
\begin{align*}
  \min_{\bm{w},\bm{\xi},\bm{\eta}} &
  -\frac{\pi_\s}{2n_\s(2\pi_+ - 1)} \bm{1}^\top X_\s\bm{w}
  -\frac{\pi_-}{n_\s(2\pi_+ - 1)} \bm{1}^\top \bm{\xi} \\
  &+\frac{\pi_+}{n_\u(2\pi_+ - 1)} \bm{1}^\top \bm{\eta}
  +\frac{\lambda}{2} \bm{w}^\top\bm{w} \\
  \text{s.t.} \qquad
  & \bm{\xi} \ge \bm{0}, \quad
    \bm{\xi} \ge \frac{1}{2}\bm{1} + \frac{1}{2}X_\u\bm{w}, \quad
    \bm{\xi} \ge X_\u\bm{w}, \\
  & \bm{\eta} \ge \bm{0}, \quad
    \bm{\eta} \ge \frac{1}{2}\bm{1} - \frac{1}{2}X_\u\bm{w}, \quad
    \bm{\eta} \ge -X_\u\bm{w},
\end{align*}
\allowdisplaybreaks
where $\ge$ for vectors denotes the element-wise inequality.
This optimization problem is a quadratic program (QP).
The transformation into the standard QP form is given in Appendix~\ref{sec:optimization-problem}.

\section{Relation between Class-Prior and SU Classification}

In Section~\ref{sec:su-learning}, we assume that the class-prior $\pi_+$ is given in advance.
In this section, we first clarify the relation between behaviors of the proposed method and $\pi_+$,
then we propose an algorithm to estimate $\pi_+$ in case we do not have $\pi_+$ in advance.

\subsection{Class-Prior-Dependent Behaviors of Proposed Method}
\label{sec:behaviors-of-su-learning}

We discuss the following three different cases on prior knowledge of $\pi_+$ (summarized in Table~\ref{tab:behaviors-of-su-learning}).

\begin{table}[t]
  \caption{
    Behaviors of the proposed method on class identification and class separation, depending on prior knowledge of the class-prior.
  }
  \label{tab:behaviors-of-su-learning}
  \centering
  \begin{tabular}{cc|cc} \hline
    Case & Prior knowledge & Identification & Separation \\ \hline
    1 & exact $\pi_+$ & \cmark & \cmark \\
    2 & nothing & \xmark & \cmark \\
    3 & $\mathrm{sign}(\pi_+ - \pi_-)$ & \cmark & \cmark \\
    \hline
  \end{tabular}
  \reducespace
\end{table}

\textbf{(Case 1) The class-prior is given:}
In this case, we can directly solve the optimization problem~\eqref{eq:optimization-problem}.
The solution does not only separate data but also \emph{identifies classes}, i.e., determine which class is positive.

\textbf{(Case 2) No prior knowledge on the class-prior is given:}
In this case, we need to estimate $\pi_+$ before solving~\eqref{eq:optimization-problem}.
If we assume $\pi_+ > \pi_-$, the estimation method in Section~\ref{sec:class-prior-estimation} gives an estimator of $\pi_+$.
Thus, we can regard the larger cluster as positive class and solve the optimization problem~\eqref{eq:optimization-problem}.
This time the solution just separates data because we have no prior information for class identifiability.

\textbf{(Case 3) Magnitude relation of the class-prior is given:}
Finally, consider the case where we know \emph{which class has a larger class-prior}.
In this case, we also need to estimate $\pi_+$, but surprisingly, we can identify classes.
For example, if the negative class has a larger class-prior, first we estimate the class-prior (let $\hat{\pi}$ be an estimated value).
Since Algorithm~\ref{alg:prior-estimation} given in Sec.~\ref{sec:class-prior-estimation} always gives an estimate of the class-prior of the larger class,
the positive class-prior is given as $\pi_+ = 1 - \hat{\pi}$.
After that, it reduces to Case 1.

\emph{Remark:} In all of the three cases above, our proposed method gives an \emph{inductive model},
which is applicable to out-of-sample prediction without any modification.
On the other hand, most of the unsupervised/semi-supervised clustering methods are designed for in-sample prediction,
which can only give predictions for data at hand given in advance.

\subsection{Class-Prior Estimation from Pairwise Similarity and Unlabeled Data}
\label{sec:class-prior-estimation}

We propose a class-prior estimation algorithm only from SU data.
First, let us begin with connecting the pairwise marginal distribution $p(\bm{x},\bm{x}')$ and $p_\s(\bm{x},\bm{x}')$
when two examples $\bm{x}$ and $\bm{x}'$ are drawn independently:
\begin{align}
  p&(\bm{x},\bm{x}')
  = p(\bm{x})p(\bm{x}')
  \nonumber \\
  &= \pi_+^2p_+(\bm{x})p_+(\bm{x}') + \pi_-^2p_-(\bm{x})p_-(\bm{x}') \nonumber \\
  & \qquad \pi_+\pi_-p_+(\bm{x})p_-(\bm{x}') + \pi_+\pi_-p_-(\bm{x})p_+(\bm{x}')
  \nonumber \\
  &= \pi_\s p_\s(\bm{x},\bm{x}') + \pi_\d p_\d(\bm{x},\bm{x}'),
  \label{eq:pairwise-marginal}
\end{align}
where Eq.~\eqref{eq:pairwise-similarity-conditional} was used to derive the last line,
$\pi_\d \triangleq 2\pi_+\pi_-$,
and
\noallowdisplaybreaks
\begin{align}
  p&_\mathrm{D}(\bm{x},\bm{x}') \nonumber
  \\
  &= p(\bm{x},\bm{x}' | (y=+1 \wedge y'=-1) \vee (y=-1 \wedge y'=+1)) \nonumber
  \\
  &= \frac{\pi_+\pi_-p_+(\bm{x})p_-(\bm{x}') + \pi_+\pi_-p_-(\bm{x})p_+(\bm{x}')}{2\pi_+\pi_-}
  \label{eq:pairwise-dissimilarity-conditional}
  .
\end{align}
\allowdisplaybreaks
Marginalizing out $\bm{x}'$ in Eq.~\eqref{eq:pairwise-marginal} as Lemma~\ref{lem:pointwise-similar-conditional}, we obtain
\vs{-5pt}
\begin{align}
  p(\bm{x}) = \pi_\s \tilde{p}_\s(\bm{x}) + \pi_\mathrm{D}\tilde{p}_\mathrm{D}(\bm{x}), \nonumber
\end{align}
where $\tilde{p}_\s$ is defined in Eq.~\eqref{eq:pointwise-similar-conditional}
and $\tilde{p}_\mathrm{D}(\bm{x}) \triangleq (p_+(\bm{x}) + p_-(\bm{x}))/2$.
Since we have samples $\mathcal{D}_\u$ and $\tilde{\mathcal{D}}_\s$ drawn from $p$ and $\tilde{p}_\s$ respectively (see Eqs.~\eqref{eq:marginal} and \eqref{eq:pointwise-similar-conditional}),
we can estimate $\pi_\s$ by mixture proportion estimation\footnote{
  Given a distribution $F$ which is a convex combination of distributions $G$ and $H$ such that $F = (1-\kappa)G + \kappa H$,
  the mixture proportion estimation problem is to estimate $\kappa\in[0,1]$ only with samples from $F$ and $H$.
  In our case, $F$, $H$, and $\kappa$ correspond to $p(\bm{x})$, $\tilde{p}_\s(\bm{x})$, and $\pi_\s$, respectively.
  See, e.g., \citet{Scott:2015}.
} methods \cite{Scott:2015,Ramaswamy:2016,Plessis:2017}.

After estimating $\pi_\s$, we can calculate $\pi_+$.
By the discussion in Section~\ref{sec:behaviors-of-su-learning}, we assume $\pi_+ > \pi_-$.
Then, following
$2\pi_\s - 1 = \pi_\s - \pi_\d = (\pi_+ - \pi_-)^2 = (2\pi_+ - 1)^2 \ge 0$,
we obtain $\pi_+ = \frac{\sqrt{2\pi_\s - 1} + 1}{2}$.
We summarize a wrapper of mixture proportion estimation in Algorithm~\ref{alg:prior-estimation}.

\begin{algorithm}[t]
  \caption{Prior estimation from SU data. $\texttt{CPE}$ is a class-prior estimation algorithm.}
  \label{alg:prior-estimation}
  \begin{algorithmic}
    \REQUIRE $\mathcal{D}_\u=\{\bm{x}_{\u,i}\}_{i=1}^{n_\u}$ (samples from $p$), $\tilde{\mathcal{D}}_\s = \{\tilde{\bm{x}}_{\s,i}\}_{i=1}^{2n_\s}$ (samples from $\tilde{p}_\s$)
    \ENSURE class-prior $\pi_+$
    \STATE $\pi_\s \leftarrow \texttt{CPE}(\mathcal{D}_\u, \tilde{\mathcal{D}}_\s)$
    \STATE $\pi_+ \leftarrow \frac{\sqrt{2\pi_\s - 1} + 1}{2}$
  \end{algorithmic}
\end{algorithm}

\section{Estimation Error Bound}
\label{sec:theory}

In this section, we establish an estimation error bound for the proposed method.
Hereafter, let $\mathcal{F} \subset \mathbb{R}^\mathcal{X}$ be a function class of a specified model.

\begin{definition}
  Let $n$ be a positive integer,
  $Z_1,\dots,Z_n$ be i.i.d.~random variables drawn from a probability distribution with density $\mu$,
  $\mathcal{H} = \{h: \mathcal{Z} \rightarrow \mathbb{R}\}$ be a class of measurable functions,
  and $\bm{\sigma}=(\sigma_1,\dots,\sigma_n)$ be Rademacher variables,
  i.e., random variables taking $+1$ and $-1$ with even probabilities.
  Then (expected) Rademacher complexity of $\mathcal{H}$ is defined as
  \begin{align*}
    \mathfrak{R}(\mathcal{H};n,\mu) \triangleq \expect_{Z_1,\dots,Z_n\sim\mu} \expect_{\bm{\sigma}} \left[ \sup_{h \in \mathcal{H}} \frac{1}{n}\sum_{i=1}^n \sigma_i h(Z_i) \right]
    .
  \end{align*}
\end{definition}

In this section, we assume for any probability density $\mu$, our model class $\mathcal{F}$ satisfies
\begin{align}
  \mathfrak{R}(\mathcal{F};n,\mu) \le \frac{C_\mathcal{F}}{\sqrt{n}} \label{eq:model-assumption}
\end{align}
for some constant $C_\mathcal{F} > 0$.
This assumption is reasonable because many model classes such as the linear-in-parameter model class
$\mathcal{F} = \{f(\bm{x}) = \bm{w}^\top\bm{\phi}(\bm{x}) \mid \|\bm{w}\| \le C_{\bm{w}}, \|\bm{\phi}\|_\infty \le C_{\bm{\phi}}\}$
($C_{\bm{w}}$ and $C_{\bm{\phi}}$ are positive constants) satisfy it~\cite{Mohri:2012}.

Subsequently, let
$f^* \triangleq \argmin_{f \in \mathcal{F}} R(f)$
be the true risk minimizer, and
$\hat{f} \triangleq \argmin_{f \in \mathcal{F}} \hat{R}_{\mathrm{SU},\ell}(f)$
be the empirical risk minimizer.

\begin{theorem}
  \label{thm:estimation-error-bound}
  Assume the loss function $\ell$ is $\rho$-Lipschitz with respect to the first argument ($0 < \rho < \infty$),
  and all functions in the model class $\mathcal{F}$ are bounded, i.e.,
  there exists a constant $C_b$ such that $\|f\|_\infty \le C_b$ for any $f \in \mathcal{F}$.
  Let $C_\ell \triangleq \sup_{t\in\{\pm 1\}} \ell(C_b,t)$.
  For any $\delta > 0$, with probability at least $1 - \delta$,
  \begin{align}
    R(\hat{f}) - R(f^*) \le C_{\mathcal{F},\ell,\delta}\left(
      \frac{2\pi_\s}{\sqrt{2n_\s}} + \frac{1}{\sqrt{n_\u}}
    \right)
    ,
    \label{eq:bound}
  \end{align}
  \vs{-3pt}
  where
  \vs{-5pt}
  \begin{align*}
    C_{\mathcal{F},\ell,\delta} &= \frac{4\rho C_\mathcal{F} + \sqrt{2C_\ell^2\log\frac{4}{\delta}}}{|2\pi_+ - 1|} .
  \end{align*}
\end{theorem}
A proof is given in Appendix~\ref{sec:proof-of-estimation-error-bound}.

Theorem~\ref{thm:estimation-error-bound} shows that if we have $\pi_+$ in advance, our proposed method is consistent,
i.e., $R(\hat{f}) \rightarrow R(f^*)$ as $n_\s \rightarrow \infty$ and $n_\u \rightarrow \infty$.
The convergence rate is $\mathcal{O}_p(1/\sqrt{n_\s} + 1/\sqrt{n_\u})$, where $\mathcal{O}_p$ denotes the order in probability.
This order is the optimal parametric rate for the empirical risk minimization without additional assumptions~\cite{Mendelson:2008}.

\section{Experiments}
\label{sec:experiments}

In this section, we empirically investigate the performance of class-prior estimation and the proposed method for SU classification.

\textbf{Datasets:}
Datasets are obtained from the {\it UCI Machine Learning Repository}~\cite{Lichman:2013}, the {\it LIBSVM}~\cite{CC01a}, and the {\it ELENA project}~\footnote{\url{https://www.elen.ucl.ac.be/neural-nets/Research/Projects/ELENA/elena.htm}}.
We randomly subsample the original datasets, to maintain that similar pairs consist of positive and negative pairs with the ratio of $\pi_+^2$ to $\pi_-^2$ (see Eq.~\eqref{eq:pairwise-similarity-conditional}), while the ratios of unlabeled and test data are $\pi_+$ to $\pi_-$ (see Eq.~\eqref{eq:marginal}).

\subsection{Class-Prior Estimation}

First, we study empirical performance of class-prior estimation.
We conduct experiments on benchmark datasets.
Different dataset sizes $\{200, 400, 800, 1600\}$ are tested, where half of the data are S pairs and the other half are U data.

In Figure~\ref{fig:prior-estimation}, KM1 and KM2 are plotted,
which are proposed by \citet{Ramaswamy:2016}.
We used them as \texttt{CPE} in Algorithm~\ref{alg:prior-estimation} \footnote{We used the author's implementations published in \url{http://web.eecs.umich.edu/~cscott/code/kernel_CPE.zip}.}.
Since $\pi_\s = \pi_+^2 + \pi_-^2 = 2(\pi_+ - \frac{1}{2})^2 + \frac{1}{2} \ge \frac{1}{2}$, we use additional heuristic to set $\lambda_{\mathrm{left}} = 2$ in Algorithm 1 of \citet{Ramaswamy:2016}.

\begin{figure}[t]
  \centering
  \includegraphics[width=0.48\columnwidth]{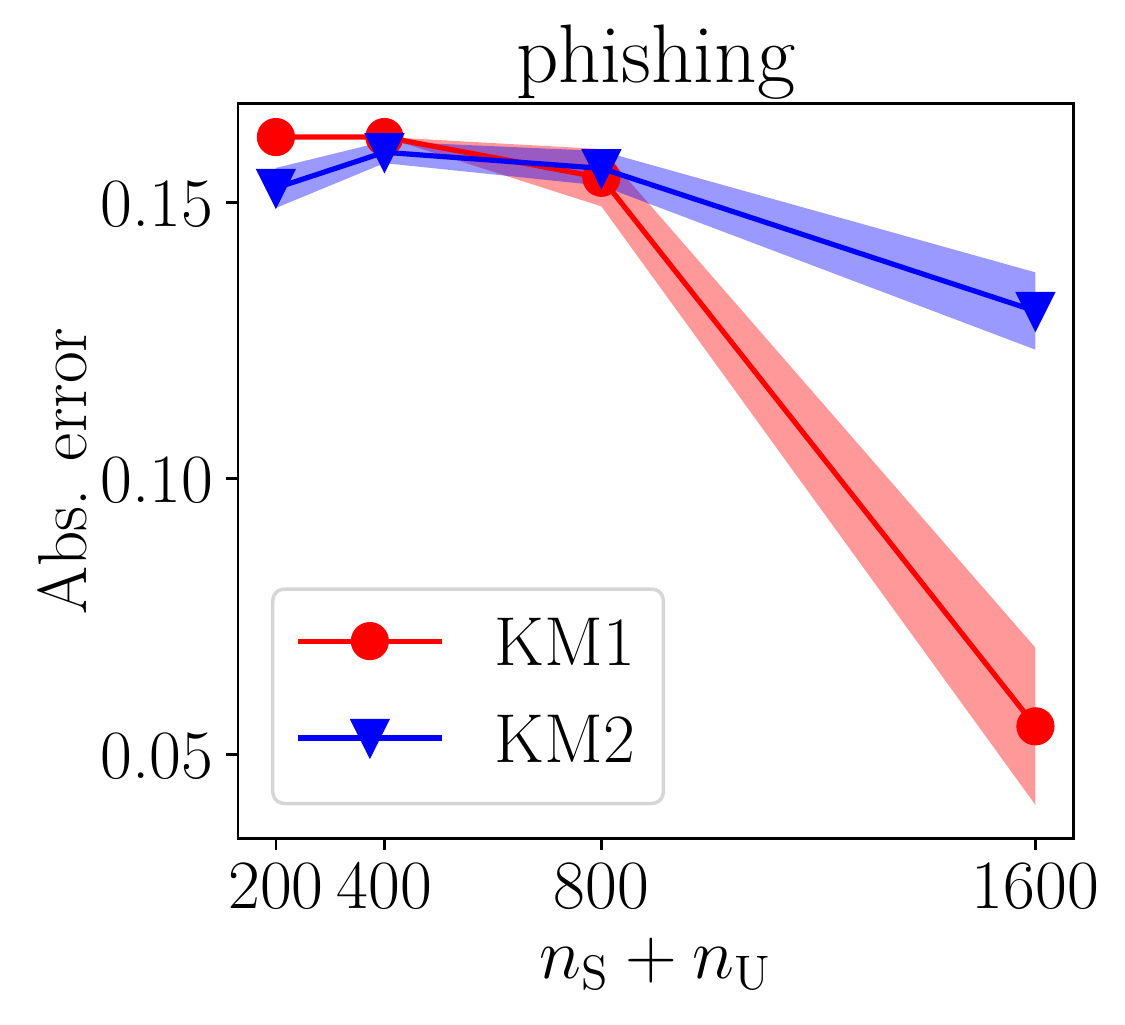}
  \includegraphics[width=0.48\columnwidth]{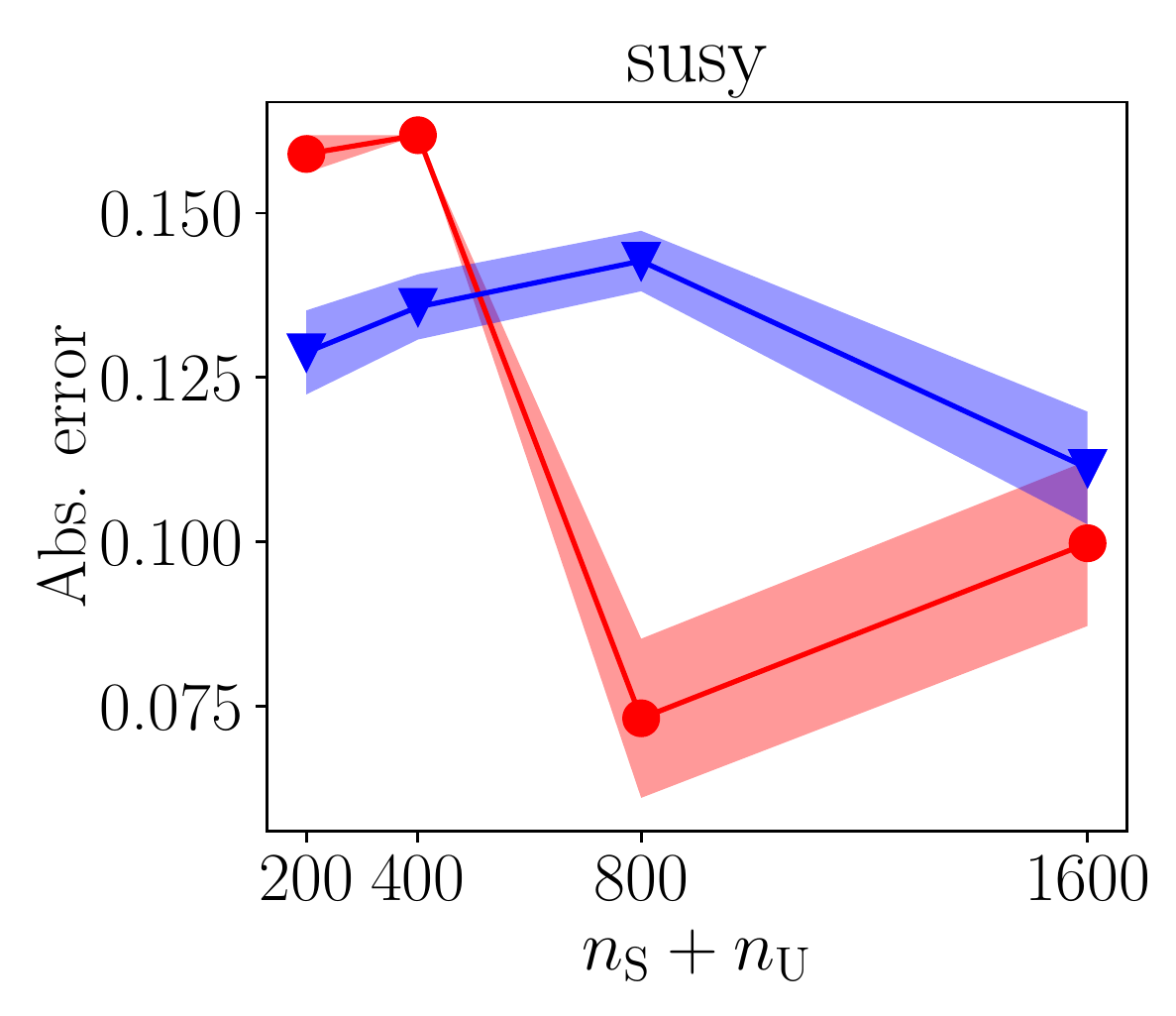}
  \caption{
    Estimation errors of the class-prior (absolute value of difference between true class-prior and estimated class-prior) from SU data over 100 trials are plotted in the vertical axes.
    For all experiments, true class-prior $\pi_+$ is set to $0.7$.
  }
  \label{fig:prior-estimation}
  \reducespace
\end{figure}

\subsection{Classification Complexity}
\label{sec:sample-complexity}

\begin{figure*}[t]
  \centering
  \begin{minipage}{0.24\hsize}
    \centering
    \includegraphics[width=\columnwidth]{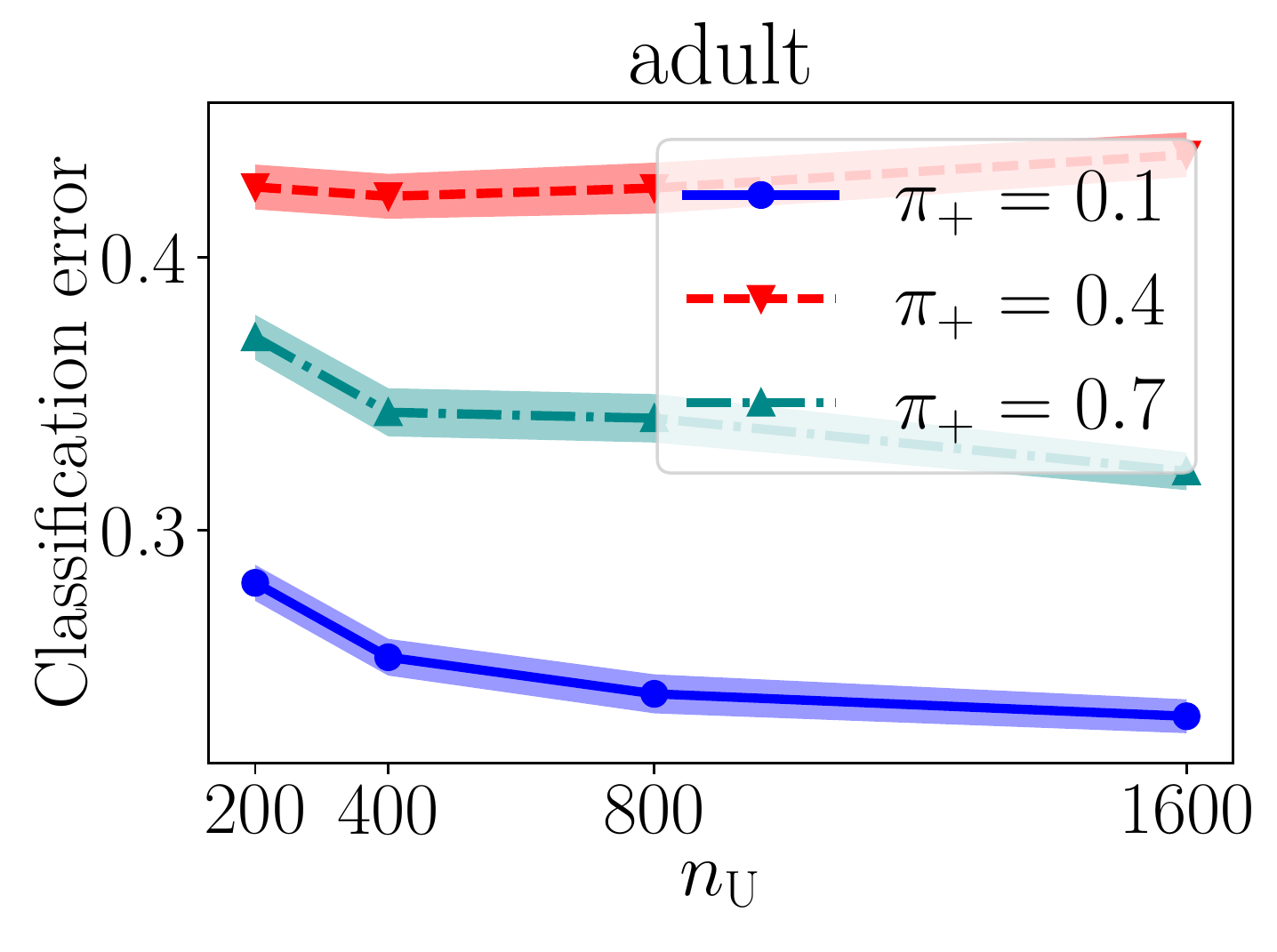}
  \end{minipage}
  \begin{minipage}{0.24\hsize}
    \centering
    \includegraphics[width=\columnwidth]{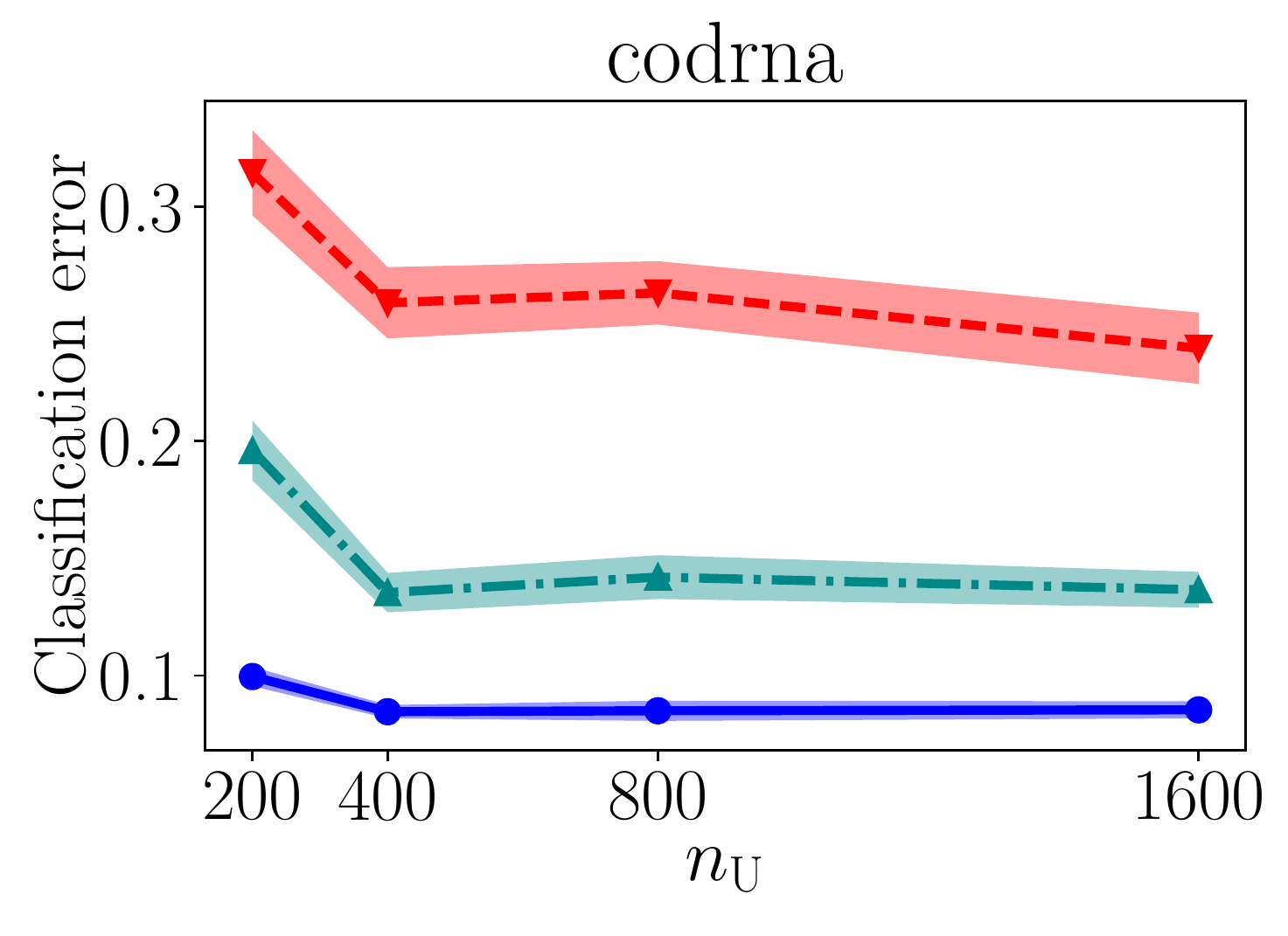}
  \end{minipage}
  \begin{minipage}{0.24\hsize}
    \centering
    \includegraphics[width=\columnwidth]{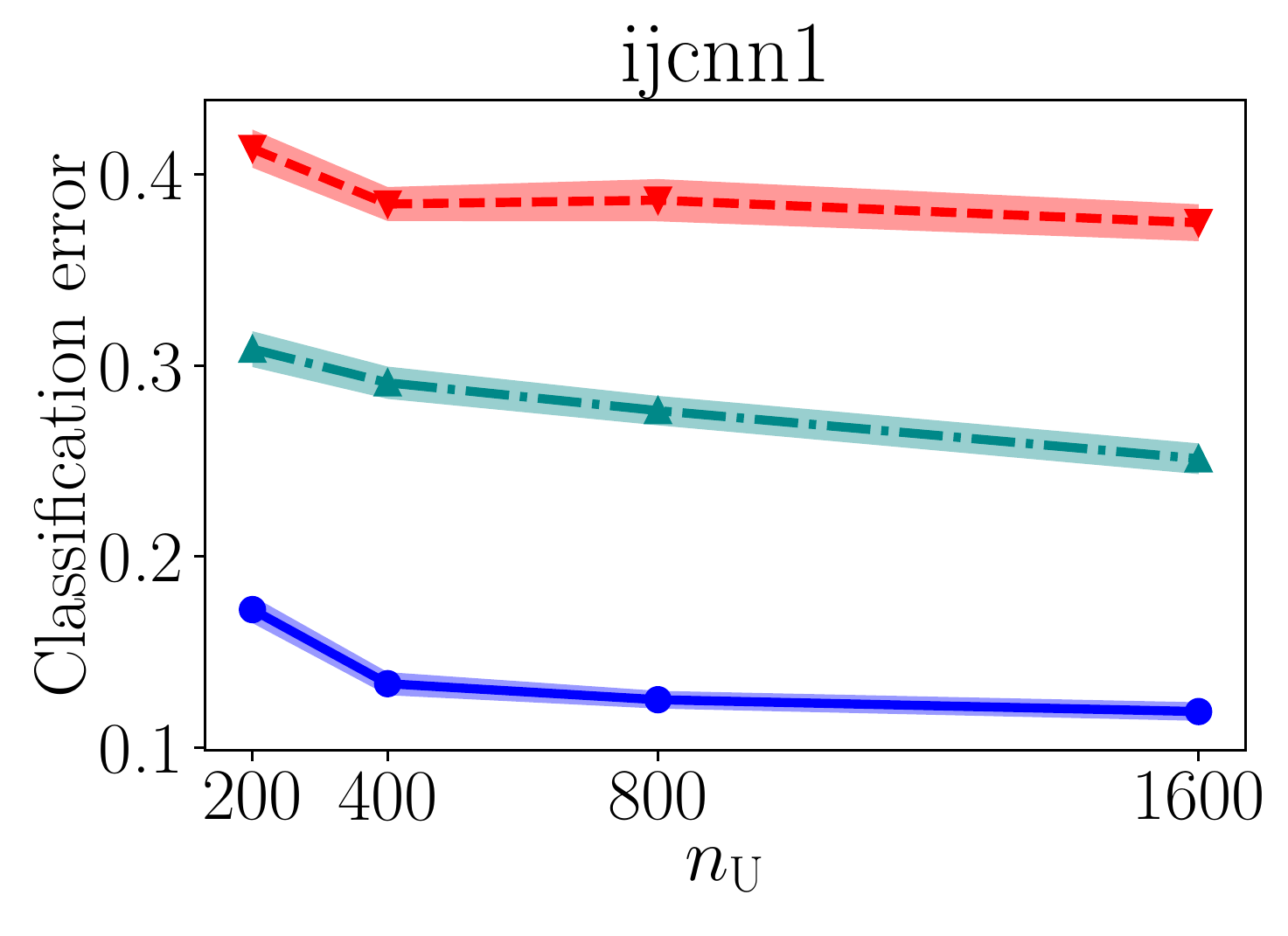}
  \end{minipage}
  \begin{minipage}{0.24\hsize}
    \centering
    \includegraphics[width=\columnwidth]{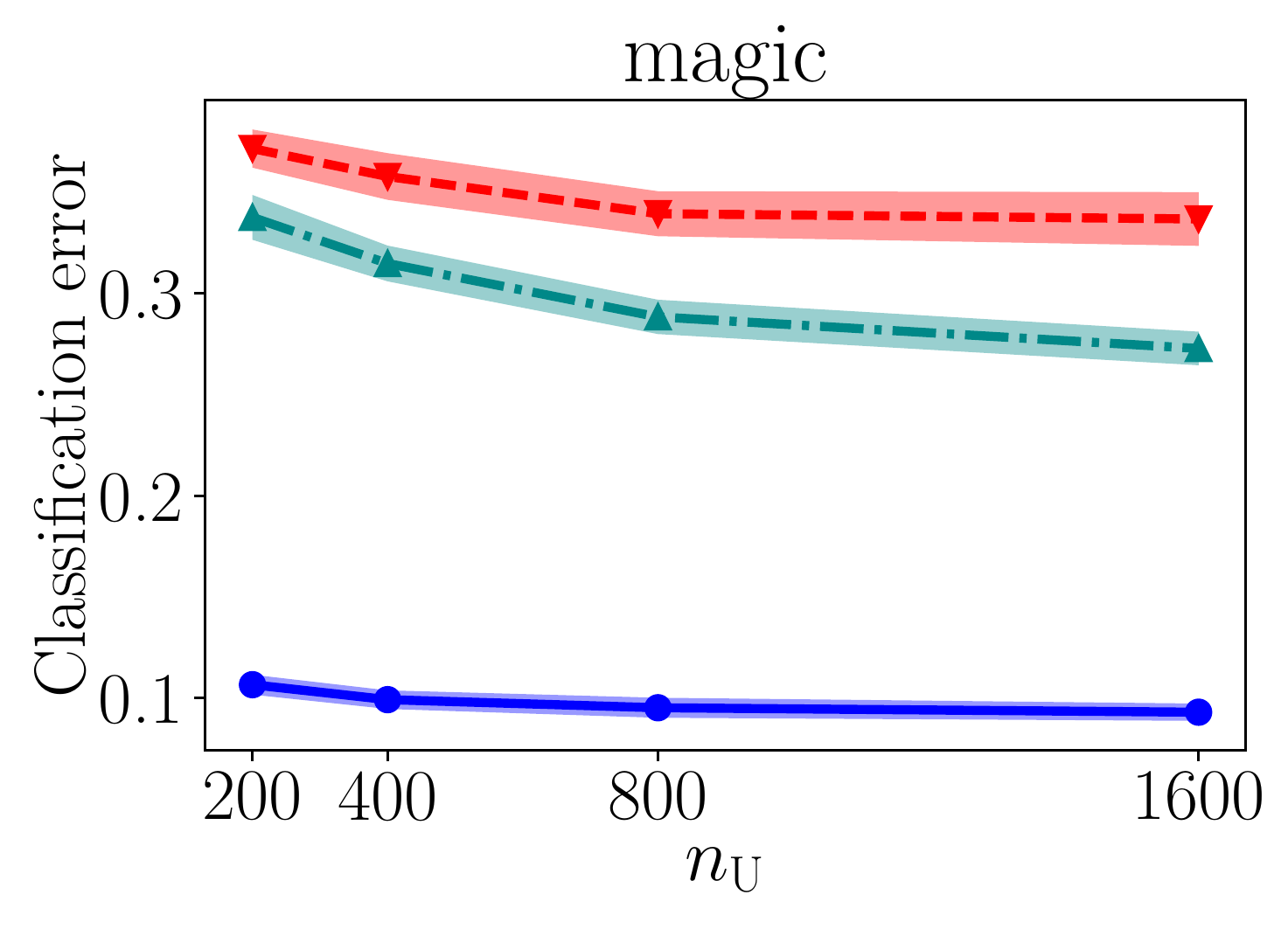}
  \end{minipage}
  \begin{minipage}{0.24\hsize}
    \centering
    \includegraphics[width=\columnwidth]{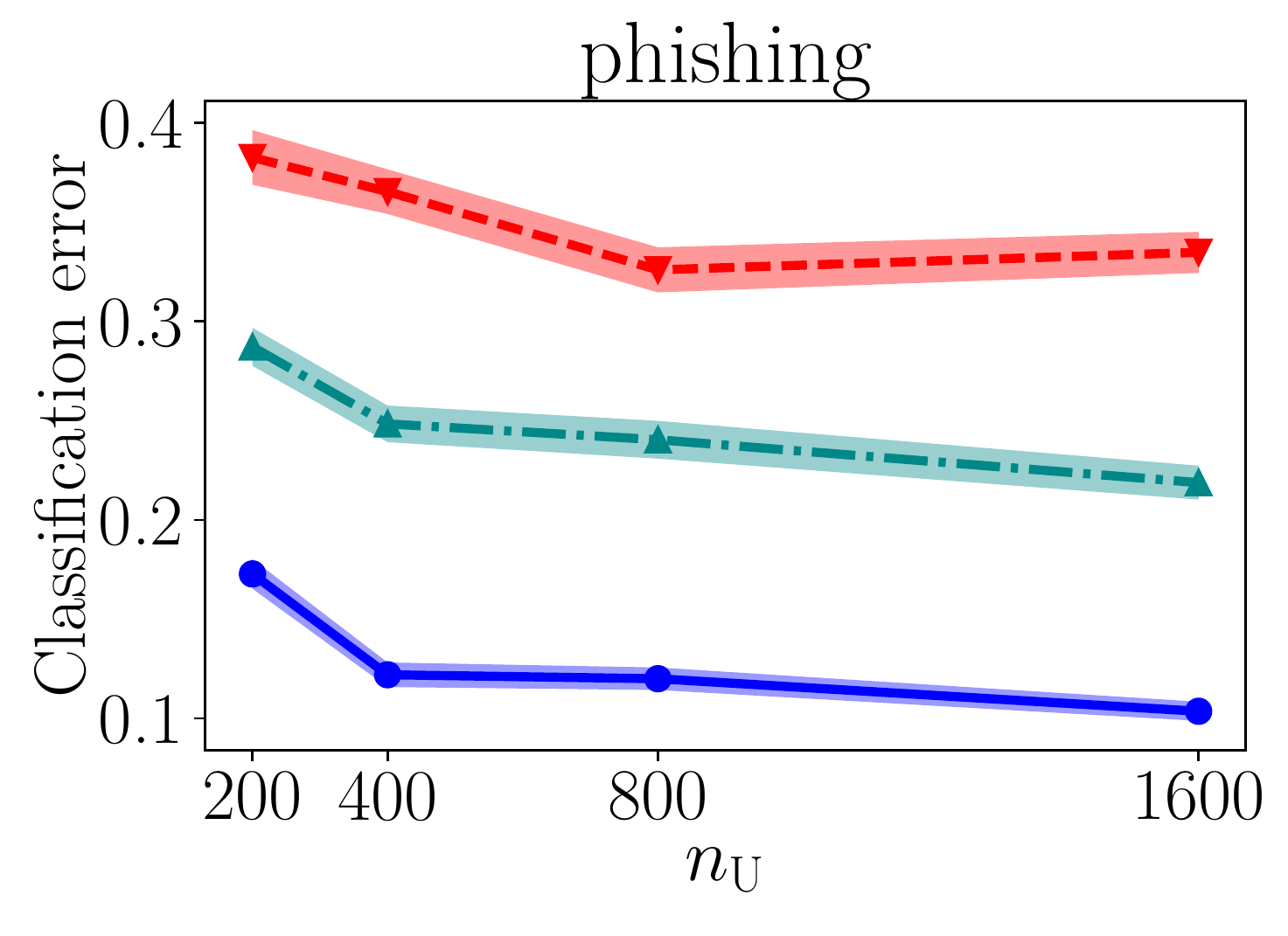}
  \end{minipage}
  \begin{minipage}{0.24\hsize}
    \centering
    \includegraphics[width=\columnwidth]{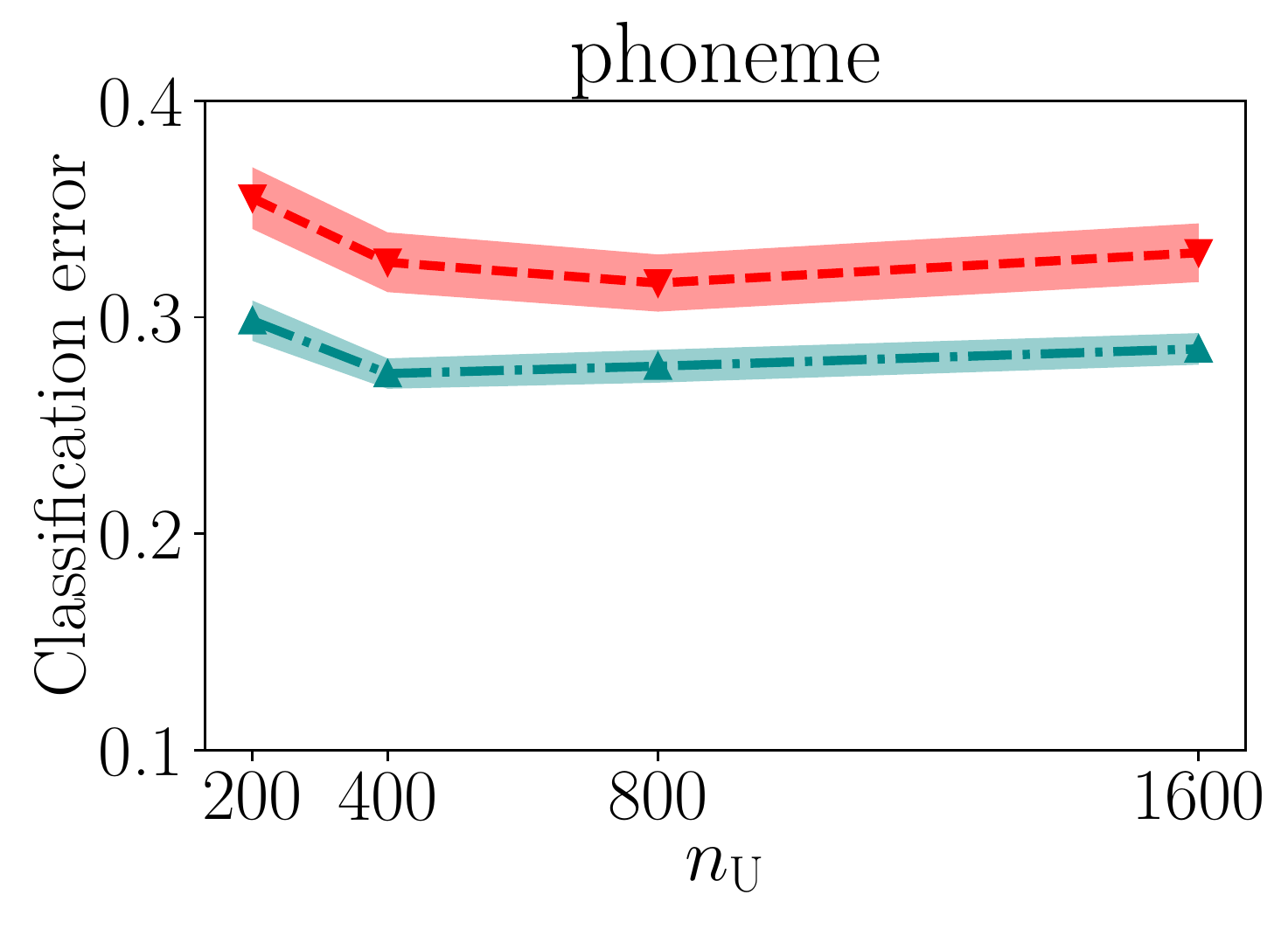}
  \end{minipage}
  \begin{minipage}{0.24\hsize}
    \centering
    \includegraphics[width=\columnwidth]{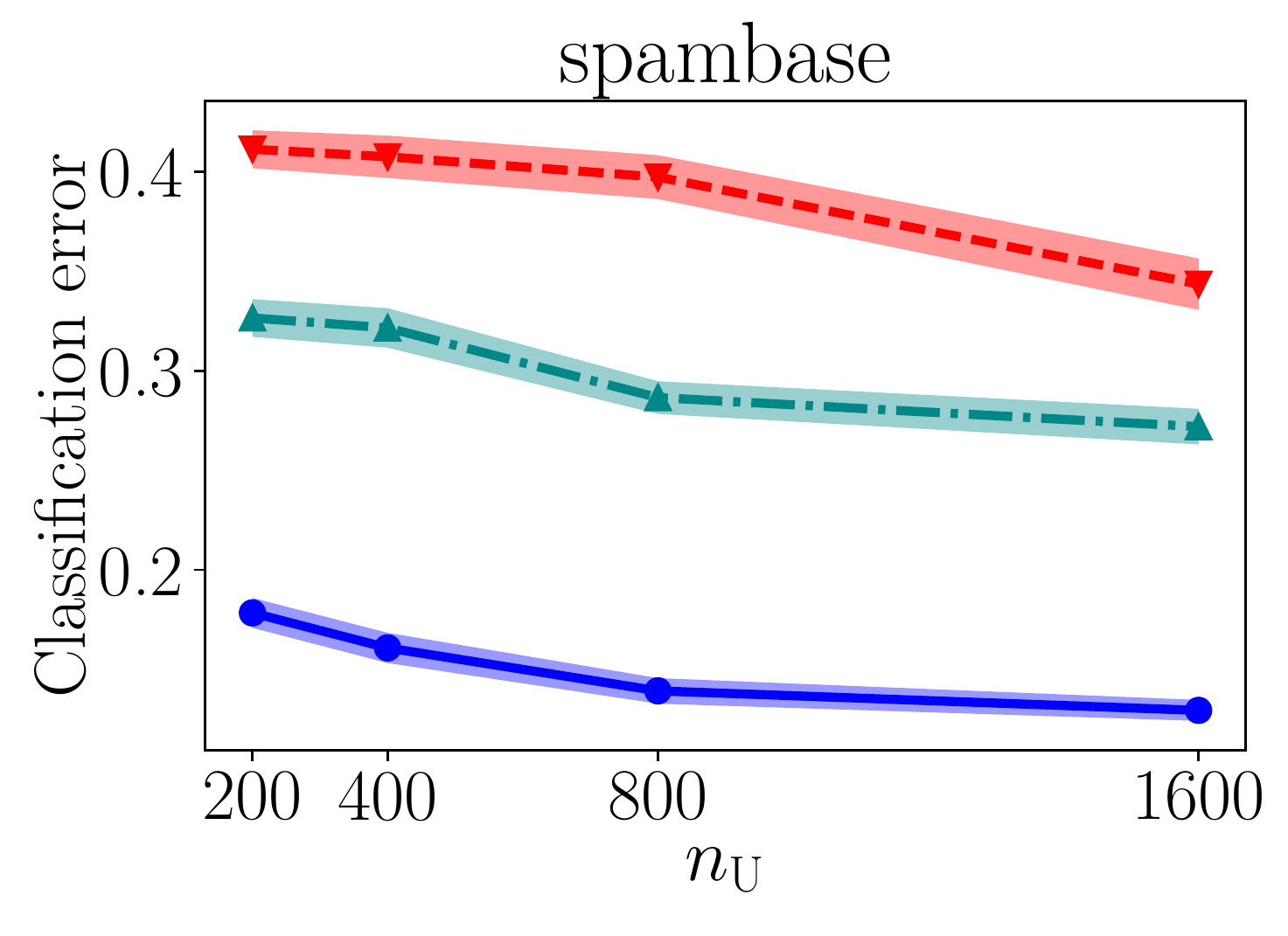}
  \end{minipage}
  \begin{minipage}{0.24\hsize}
    \centering
    \includegraphics[width=\columnwidth]{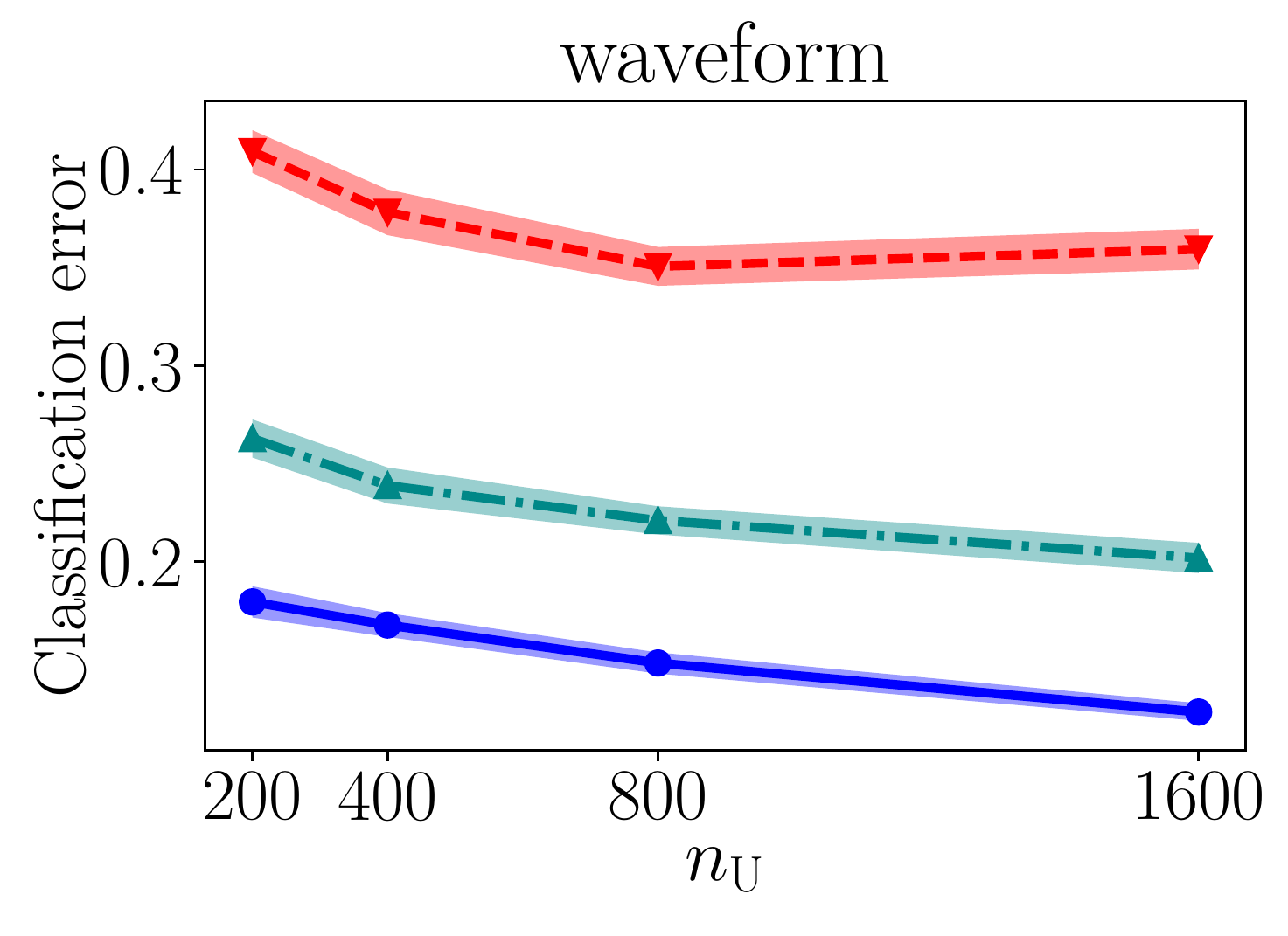}
  \end{minipage}
  \caption{
    Average classification error (vertical axes) and standard error (shaded areas) over 50 trials.
    Different $n_\u \in \{200, 400, 800, 1600\}$ are tested, while $n_\s$ is fixed to $200$.
    For each dataset, results with different class-priors ($\pi_+ \in \{0.1, 0.4, 0.7\}$) are plotted, which is assumed to be known in advance.
    Dataset ``phoneme'' does not have a plot for $\pi_+ = 0.1$ because the number of data in the original dataset is insufficient to subsample SU dataset with $\pi_+ = 0.1$.
  }
  \label{fig:nu-shift}
\end{figure*}
We empirically investigate our proposed method in terms of the relationship between classification performance and the number of training data.
We conduct experiments on benchmark datasets with the fixed number of S pairs (fixed to 200),
and the different numbers of U data $\{200, 400, 800, 1600\}$.

The experimental results are shown in Figure~\ref{fig:nu-shift}.
It indicates that the classification error decreases as $n_\u$ grows, which well agree with our theoretical analysis in Theorem~\ref{thm:estimation-error-bound}.
Furthermore, we observe a tendency that classification error becomes smaller as the class-prior becomes farther from $\tfrac{1}{2}$.
This is because $C_{\mathcal{F},\ell,\delta}$ in Eq.~\eqref{eq:bound} has the term $|2\pi_+ - 1|$ in the denominator, which will make the upper bound looser when $\pi_+$ is close to $\tfrac{1}{2}$.

The detailed setting about the proposed method is described below.
Our implementation is available at \url{https://github.com/levelfour/SU_Classification}.

\textbf{Proposed Method (SU):}
We use the linear-in-input model $f(\bm{x}) = \bm{w}^\top \bm{x} + b$.
In Section~\ref{sec:sample-complexity}, the squared loss is used,
and $\pi_+$ is given (Case 1 in Table~\ref{tab:behaviors-of-su-learning}).
In Section~\ref{sec:comparison-with-baseline-methods}, the squared loss and the double-hinge loss are used,
and the class-prior is estimated by Algorithm~\ref{alg:prior-estimation} with KM2~\cite{Ramaswamy:2016} (Case 2 in Table~\ref{tab:behaviors-of-su-learning}).
The regularization parameter $\lambda$ is chosen from $\{10^{-1}, 10^{-4}, 10^{-7}\}$.

To choose hyperparameters, 5-fold cross-validation is used.
Since we do not have any labeled data in the training phase, the validation error cannot be computed directly.
Instead, Eq.~\eqref{eq:su-estimator} equipped with the zero-one loss $\ell_{01}(\cdot) = \frac{1}{2}(1 - \sign(\cdot))$ is used as a proxy to estimate the validation error.
In each experimental trial, the model with minimum validation error is chosen.

\subsection{Benchmark Comparison with Baseline Methods}
\label{sec:comparison-with-baseline-methods}

\begin{table*}[t]
  \centering
  \caption{
    Mean accuracy and standard error of SU classification on different benchmark datasets over 20 trials.
    For all experiments, class-prior $\pi_+$ is set to $0.7$.
    The proposed method does not have oracle $\pi_+$ in advance, instead estimating it.
    Performances are measured by the clustering accuracy $1 - \min(r, 1-r)$, where $r$ is error rate.
    Bold-faces indicate outperforming methods, chosen by one-sided t-test with the significance level $5\%$.
    The result of SERAPH with ``w8a'' is unavailable due to high-dimensionality and memory constraints.
  }
  \label{tab:benchmark}
  \scalebox{0.9}{
  \begin{tabular}{cccccccccc} \hline
    & & \multicolumn{2}{c}{SU(proposed)} & \multicolumn{6}{c}{Baselines} \\ \cmidrule(lr){3-4} \cmidrule(lr){5-10}
    Dataset & Dim & Squared & Double-hinge & KM & ITML & SERAPH & 3SMIC & DIMC & IMSAT(linear) \\ \hline
    adult         & 123  &         64.5 (1.2)  & \textbf{84.5 (0.8)} &         58.1 (1.1)  &         57.9 (1.1)  &         66.5 (1.7)  &         58.5 (1.3)  &         63.7 (1.2)  &         69.8 (0.9) \\
    banana        &   2  & \textbf{67.5 (1.2)} & \textbf{68.2 (1.2)} &         54.3 (0.7)  &         54.8 (0.7)  &         55.0 (1.1)  &         61.9 (1.2)  &         64.3 (1.0)  & \textbf{69.8 (0.9)}\\
    cod-rna       &   8  & \textbf{82.8 (1.3)} &         71.0 (0.9)  &         63.1 (1.1)  &         62.8 (1.0)  &         62.5 (1.4)  &         56.6 (1.2)  &         63.8 (1.1)  &         69.1 (0.9) \\
    higgs         &  28  &         55.1 (1.1)  & \textbf{69.3 (0.9)} & \textbf{66.4 (1.6)} & \textbf{66.6 (1.3)} &         63.4 (1.1)  &         57.0 (0.9)  &         65.0 (1.1)  & \textbf{69.7 (1.4)}\\
    ijcnn1        &  22  &         65.5 (1.3)  & \textbf{73.6 (0.9)} &         54.6 (0.9)  &         55.8 (0.7)  &         59.8 (1.2)  &         58.9 (1.3)  &         66.2 (2.2)  &         68.5 (1.1) \\
    magic         &  10  &         66.0 (2.0)  & \textbf{69.0 (1.3)} &         53.9 (0.6)  &         54.5 (0.7)  &         55.0 (0.9)  &         59.1 (1.4)  &         63.1 (1.1)  & \textbf{70.0 (1.1)}\\
    phishing      &  68  &         75.0 (1.4)  & \textbf{91.3 (0.6)} &         64.4 (1.0)  &         61.9 (1.1)  &         62.4 (1.1)  &         60.1 (1.3)  &         64.8 (1.4)  &         69.4 (0.8) \\
    phoneme       &   5  & \textbf{67.8 (1.5)} & \textbf{70.8 (1.0)} &         65.2 (0.9)  &         66.7 (1.4)  & \textbf{69.1 (1.4)} &         61.3 (1.1)  &         64.5 (1.2)  & \textbf{69.2 (1.1)}\\
    spambase      &  57  &         69.7 (1.4)  & \textbf{85.5 (0.5)} &         60.1 (1.8)  &         54.4 (1.1)  &         65.4 (1.8)  &         61.5 (1.3)  &         63.6 (1.3)  &         70.5 (1.1) \\
    susy          &  18  &         59.8 (1.3)  & \textbf{74.8 (1.2)} &         55.6 (0.7)  &         55.4 (0.9)  &         58.0 (1.0)  &         57.1 (1.2)  &         65.2 (1.0)  &         70.4 (1.2) \\
    w8a           & 300  &         62.1 (1.5)  & \textbf{86.5 (0.6)} &         71.0 (0.8)  &         69.5 (1.5)  &         0.0 (0.0)  &         60.5 (1.5)  &         65.0 (2.0)  &         70.2 (1.2)  \\
    waveform      &  21  &         77.8 (1.3)  & \textbf{87.0 (0.5)} &         56.1 (0.8)  &         54.8 (0.7)  &         56.5 (0.9)  &         56.5 (0.9)  &         65.0 (0.9)  &         69.7 (1.1) \\ \hline
  \end{tabular}
  }
\end{table*}

We compare our proposed method with baseline methods on benchmark datasets.
We conduct experiments on each dataset with 500 similar data pairs, 500 unlabeled data, and 100 test data.
As can be seen from Table~\ref{tab:benchmark}, our proposed method outperforms baselines for many datasets.
The details about the baseline methods are described below.

\textbf{Baseline 1 (KM):}
As a simple baseline, we consider $k$-means clustering~\cite{MacQueen:1967}.
We ignore pair information of S data and apply $k$-means clustering with $k=2$ to U data.

\textbf{Baseline 2 (ITML):}
Information-theoretic metric learning~\cite{Davis:2007} is a metric learning method by regularizing the covariance matrix based on prior knowledge, with pairwise constraints.
We use the identity matrix as prior knowledge, and the slack variable parameter is fixed to $\gamma = 1$, since we cannot employ the cross-validation without any class label information.
Using the obtained metric, $k$-means clustering is applied on test data.

\textbf{Baseline 3 (SERAPH):}
Semi-supervised metric learning paradigm with hyper sparsity~\cite{Niu:2012} is another metric learning method based on entropy regularization.
Hyperparameter choice follows $\textsc{Seraph}_{\mathrm{hyper}}$.
Using the obtained metric, $k$-means clustering is applied on test data.

\textbf{Baseline 4 (3SMIC):}
Semi-supervised SMI-based clustering \cite{Calandriello:2014} models class-posteriors and maximizes mutual information between unlabeled data at hand and their cluster labels.
The penalty parameter $\gamma$ and the kernel parameter $t$ are chosen from $\{10^{-2}, 10^0, 10^2\}$ and $\{4, 7, 10\}$, respectively, via 5-fold cross-validation.

\textbf{Baseline 5 (DIMC):}
DirtyIMC~\cite{Chiang:2015} is a noisy version of inductive matrix completion, where the similarity matrix is recovered from a low-rank feature matrix.
Similarity matrix $S$ is assumed to be expressed as $UU^\top$, where $U$ is low-rank feature representations of input data.
After obtaining $U$, $k$-means clustering is conducted on $U$.
Two hyperparameters $\lambda_M, \lambda_N$ in Eq.~(2) in \cite{Chiang:2015} are set to $\lambda_M = \lambda_N = 10^{-2}$.

\textbf{Baseline 6 (IMSAT):}
Information maximizing self-augmented training~\cite{Hu:2017} is an unsupervised learning method to make a probabilistic classifier that maps similar data to similar representations,
combining information maximization clustering with self-augmented training,
which make the predictions of perturbed data close to the predictions of the original ones.
Instead of data perturbation, self-augmented training can be applied on S data to make each pair of data similar.
Here the logistic regressor $p_{\bm{\theta}}(y|\bm{x}) = (1 + \exp(-\bm{\theta}^\top\bm{x}))^{-1}$ is used as a classification model,
where $\bm{\theta}$ is parameters to learn.
Trade-off parameter $\lambda$ is set to $1$.

\textit{Remark:}
KM, ITML, and SERAPH rely on $k$-means, which is trained by using only training data.
Test prediction is based on the metric between test data and learned cluster centers.
Among the baselines, DIMC can only handle in-sample prediction, so it is trained by using both training and test data at the same time.

\section{Conclusion}
\label{sec:conclusion}

In this paper, we proposed a novel weakly-supervised learning problem named SU classification, where only similar pairs and unlabeled data are needed.
SU classification even becomes class-identifiable under a certain condition on the class-prior (see Table~\ref{tab:behaviors-of-su-learning}).
Its optimization problem with the linear-in-parameter model becomes convex if we choose certain loss functions such as the squared loss and the double-hinge loss.
We established an estimation error bound for the proposed method,
and confirmed that the estimation error decreases with the parametric optimal order,
as the number of similar data and unlabeled data becomes larger.
We also investigated the empirical performance and confirmed that our proposed method performs better than baseline methods.


\section*{Acknowledgements}
This work was supported by JST CREST JPMJCR1403 including the AIP challenge program, Japan.
We thank Ryuichi Kiryo for fruitful discussions on this work.

\bibliography{ref}
\bibliographystyle{icml2018}

\appendix
\onecolumn
\renewcommand{\thetable}{\Alph{section}.\arabic{figure}}
\renewcommand{\thefigure}{\Alph{section}.\arabic{table}}
\makeatletter
\@addtoreset{equation}{section}
\def\theequation{\thesection.\arabic{equation}}
\makeatother

\section{Proof of Lemma~\ref{lem:pointwise-similar-conditional}}
\label{sec:proof-of-pointwise-similar-conditional}

From the assumption~\eqref{eq:pairwise-similarity-conditional}, $\mathcal{D}_\s = \{(\bm{x}_{\s,i},\bm{x}_{\s,i}')\}_{i=1}^{n_\s} \sim p_\s(\bm{x},\bm{x}')$.
In order to decompose pairwise data into pointwise,
marginalize $p_\s(\bm{x},\bm{x}')$ with respect to $\bm{x}'$:
\begin{align*}
  \int p_\s(\bm{x},\bm{x}')d\bm{x}'
  &= \frac{\pi_+^2}{\pi_+^2 + \pi_-^2}p_+(\bm{x}) \int p_+(\bm{x}')d\bm{x}'
  + \frac{\pi_-^2}{\pi_+^2 + \pi_-^2}p_-(\bm{x})\int p_-(\bm{x}')d\bm{x}'
  \\
  &= \frac{\pi_+^2}{\pi_+^2 + \pi_-^2}p_+(\bm{x}) \int \frac{p(\bm{x}',y=+1)}{p(y=+1)}d\bm{x}'
  + \frac{\pi_-^2}{\pi_+^2 + \pi_-^2}p_-(\bm{x}) \int \frac{p(\bm{x}',y=-1)}{p(y=-1)}d\bm{x}'
  \\
  &= \frac{\pi_+^2}{\pi_+^2 + \pi_-^2}p_+(\bm{x}) \frac{p(y=+1)}{p(y=+1)}
  + \frac{\pi_-^2}{\pi_+^2 + \pi_-^2}p_-(\bm{x}) \frac{p(y=-1)}{p(y=-1)}
  \\
  &= \frac{\pi_+^2}{\pi_+^2 + \pi_-^2}p_+(\bm{x}) + \frac{\pi_-^2}{\pi_+^2 + \pi_-^2}p_-(\bm{x})
  \\
  &= \tilde{p}_\s(\bm{x})
  .
\end{align*}
Since a pair $(\bm{x}_{\s,i},\bm{x}_{\s,i}') \in \mathcal{D}_\s$ is independently and identically drawn, both $\bm{x}_{\s,i}$ and $\bm{x}_{\s,i}'$ are drawn following $\tilde{p}_\s$.
\qed

\section{Proof of Theorem~\ref{thm:su-risk}}
\label{sec:proof-of-su-estimator}

To prove Theorem~\ref{thm:su-risk}, it is convenient to begin with the following Lemma~\ref{lem:psd-estimator}.

\begin{lemma}
  \label{lem:psd-estimator}
  The classification risk~\eqref{eq:risk} can be equivalently expressed as
  \begin{align}
    R&_{\mathrm{PSD},\ell}(f)
    = \frac{\pi_+}{2\pi_-} \expect_{X \sim p_+} \left[ \tilde{\ell}(f(X)) \right] \nonumber \\
    & + \pi_\s \expect_{(X,X') \sim p_\s} \left[ -\frac{\pi_+}{2\pi_-}\frac{\ell(f(X),+1) + \ell(f(X'),+1)}{2} + \frac{1+\pi_-}{2\pi_-}\frac{\ell(f(X),-1) + \ell(f(X'), -1)}{2} \right] \nonumber \\
    & + \pi_\mathrm{D} \expect_{(X,X') \sim p_\mathrm{D}} \left[ \frac{\ell(f(X),-1) + \ell(f(X'),+1)}{2} \right]
    \label{eq:psd-estimator}
    ,
  \end{align}
  where $\mathbb{E}_{X \sim p_+}[\cdot]$, $\mathbb{E}_{(X,X') \sim p_\s}[\cdot]$, and $\mathbb{E}_{(X,X') \sim p_\d}[\cdot]$
  denote expectations over $p_+(X)$, $p_\s(X,X')$, and $p_\d(X,X')$, respectively.
\end{lemma}

Note that the definitions of $p_\d$ and $\pi_\d$ are first given in Eq.~\eqref{eq:pairwise-dissimilarity-conditional}.

\begin{proof}
Eq.~\eqref{eq:risk} can be transformed into pairwise fashion:
\begin{align}
  \expect_{(X,Y) \sim p} [\ell(f(X), Y)]
  &= \expect_{(X,Y) \sim p} \left[ \frac{\ell(f(X), Y)}{2} \right] + \expect_{(X',Y') \sim p} \left[ \frac{\ell(f(X'), Y')}{2} \right] \nonumber
  \\
  &= \expect_{(X,Y),(X',Y') \sim p} \left[ \frac{\ell(f(X),Y) + \ell(f(X'),Y')}{2} \right]
  \label{eq:pairwise-transform}
  .
\end{align}
Both pairs $(X,Y)$ and $(X',Y')$ are independently and identically distributed from the joint distribution $p(\bm{x},y)$.
Thus, Eq.~\eqref{eq:pairwise-transform} can be further decomposed:
\begin{align}
  \expect_{(X,Y),(X',Y') \sim p} & \left[ \frac{\ell(f(X),Y) + \ell(f(X'),Y')}{2} \right] \nonumber
  \\
  &= \sum_{y,y'} \int \frac{\ell(f(\bm{x}),y) + \ell(f(\bm{x}'),y')}{2} p(\bm{x},y)p(\bm{x}',y') d\bm{x}d\bm{x}' \nonumber
  \\
  &= \pi_+^2 \int \frac{\ell(f(\bm{x}),+1) + \ell(f(\bm{x}'),+1)}{2} p_+(\bm{x})p_+(\bm{x}') d\bm{x}d\bm{x}' \nonumber \\
  & \quad + \pi_-^2 \int \frac{\ell(f(\bm{x}),-1) + \ell(f(\bm{x}'),-1)}{2} p_-(\bm{x})p_-(\bm{x}') d\bm{x}d\bm{x}' \nonumber \\
  & \quad + \pi_+\pi_- \int \frac{\ell(f(\bm{x}),+1) + \ell(f(\bm{x}'),-1)}{2} p_+(\bm{x})p_-(\bm{x}') d\bm{x}d\bm{x}' \nonumber \\
  & \quad + \pi_+\pi_- \int \frac{\ell(f(\bm{x}),-1) + \ell(f(\bm{x}'),+1)}{2} p_-(\bm{x})p_+(\bm{x}') d\bm{x}d\bm{x}' \nonumber
  \\
  &= \pi_+^2 \expect_{X,X' \sim p_+} \left[ \frac{\ell(f(X),+1) + \ell(f(X'),+1)}{2} \right] \nonumber \\
  & \qquad + \pi_-^2 \expect_{X,X' \sim p_-} \left[ \frac{\ell(f(X),-1) + \ell(f(X'),-1)}{2} \right] \nonumber \\
  & \qquad + \pi_+ \pi_- \expect_{X \sim p_+, X' \sim p_-} \left[ \frac{\ell(f(X),+1) + \ell(f(X'),-1)}{2} \right] \nonumber \\
  & \qquad + \pi_+ \pi_- \expect_{X \sim p_-, X' \sim p_+} \left[ \frac{\ell(f(X),-1) + \ell(f(X'),+1)}{2} \right]
  .
  \label{eq:pairwise-risk-decomposition}
\end{align}
Using Eq.~\eqref{eq:pairwise-similarity-conditional}, the following equation is obtained:
\begin{align}
  \pi_\s& \expect_{(X,X') \sim p_\s} \left[ \frac{\ell(f(X),-1) + \ell(f(X'),-1)}{2} \right] \nonumber
  \\
  &= \pi_+^2 \expect_{X,X' \sim p_+} \left[ \frac{\ell(f(X),-1) + \ell(f(X'),-1)}{2} \right]
  + \pi_-^2 \expect_{X,X' \sim p_-} \left[ \frac{\ell(f(X),-1) + \ell(f(X'),-1)}{2} \right]
  .
  \label{eq:pairwise-similarity-conditional-sub}
\end{align}
Similarly, the following equation is obtained from Eq.~\eqref{eq:pairwise-dissimilarity-conditional}:
\begin{align}
  2&\pi_-\pi_+ \expect_{(X,X') \sim p_\mathrm{D}} \left[ \frac{\ell(f(X),-1) + \ell(f(X'),+1)}{2} \right] \nonumber
  \\
  &= \pi_+\pi_- \expect_{X \sim p_+, X' \sim p_-} \left[ \frac{\ell(f(X),-1) + \ell(f(X'),+1)}{2} \right]
  + \pi_+\pi_- \expect_{X \sim p_-, X \sim p_+} \left[ \frac{\ell(f(X),-1) + \ell(f(X'),+1)}{2} \right]
  .
  \label{eq:pairwise-dissimilarity-conditional-sub}
\end{align}
Combining Eqs.~\eqref{eq:pairwise-risk-decomposition},~\eqref{eq:pairwise-similarity-conditional-sub} and~\eqref{eq:pairwise-dissimilarity-conditional-sub},
the expected risk $R(f)$ is written as
\begin{align}
  R(f)
  &= \pi_\s \expect_{(X,X') \sim p_\s} \left[ \frac{\ell(f(X),-1) + \ell(f(X'),-1)}{2} \right]
  + \pi_\mathrm{D} \expect_{(X,X') \sim p_\mathrm{D}} \left[ \frac{\ell(f(X),-1) + \ell(f(X'),+1)}{2} \right] \nonumber \\
  & \quad + \pi_+^2 \expect_{X,X' \sim p_+} \left[ \frac{\tilde{\ell}(f(X)) + \tilde{\ell}(f(X'))}{2} \right]
  + \pi_+\pi_- \expect_{X \sim p_+, X' \sim p_-} \left[ \frac{\tilde{\ell}(f(X)) - \tilde{\ell}(f(X'))}{2} \right]
  .
  \label{eq:risk-decomposition-sub}
\end{align}
Here
\begin{itemize}
  \item
  the second term on the RHS of Eq.~\eqref{eq:pairwise-similarity-conditional-sub} is substituted into the second term in the last line of Eq.~\eqref{eq:pairwise-risk-decomposition}.

  \item
  the second term on the RHS of Eq.~\eqref{eq:pairwise-dissimilarity-conditional-sub} is substituted into the fourth term in the last line of Eq.~\eqref{eq:pairwise-risk-decomposition}.
\end{itemize}
On the third and fourth term on the RHS of Eq.~\eqref{eq:risk-decomposition-sub},
\begin{align}
  \pi_+^2 & \expect_{X,X' \sim p_+} \left[ \frac{\tilde{\ell}(f(X)) + \tilde{\ell}(f(X'))}{2} \right]
  + \pi_+\pi_- \expect_{X \sim p_+, X' \sim p_-} \left[ \frac{\tilde{\ell}(f(X)) - \tilde{\ell}(f(X'))}{2} \right] \nonumber
  \\
  &= \pi_+^2 \left\{ \frac{1}{2}\expect_{X \sim p_+}\left[\tilde{\ell}(f(X))\right] + \frac{1}{2}\expect_{X' \sim p_+}\left[\tilde{\ell}(f(X'))\right] \right\}
  + \pi_+\pi_- \left\{ \frac{1}{2}\expect_{X \sim p_+}\left[\tilde{\ell}(f(X))\right] - \frac{1}{2}\expect_{X' \sim p_-}\left[\tilde{\ell}(f(X'))\right] \right\} \nonumber
  \\
  &= \pi_+^2 \expect_{X \sim p_+} \left[ \tilde{\ell}(f(X)) \right] + \frac{1}{2} \pi_+\pi_- \expect_{X \sim p_+} \left[ \tilde{\ell}(f(X)) \right]
  - \frac{1}{2} \pi_+\pi_- \expect_{X \sim p_-} \left[ \tilde{\ell}(f(X)) \right] \nonumber
  \\
  &= \frac{\pi_+(1 + \pi_+)}{2} \expect_{X \sim p_+} \left[ \tilde{\ell}(f(X)) \right]
  - \frac{\pi_+(1 - \pi_+)}{2} \expect_{X \sim p_-} \left[ \tilde{\ell}(f(X)) \right] \nonumber
  \\
  &= (\ast)
  .
  \label{eq:risk-decomposition-sub2}
\end{align}
Here similarly to derivation of Eq.~\eqref{eq:pairwise-similarity-conditional-sub},
\begin{align}
  \pi_\s& \expect_{(X,X') \sim p_\s} \left[ \frac{\tilde{\ell}(f(X)) + \tilde{\ell}(f(X'))}{2} \right] \nonumber
  \\
  &= \pi_+^2 \expect_{X,X' \sim p_+} \left[ \frac{\tilde{\ell}(f(X)) + \tilde{\ell}(f(X'))}{2} \right]
  + \pi_-^2 \expect_{X,X' \sim p_-} \left[ \frac{\tilde{\ell}(f(X)) + \tilde{\ell}(f(X'))}{2} \right] \nonumber
  \\
  &= \pi_+^2 \expect_{X \sim p_+} \left[ \tilde{\ell}(f(X)) \right] + \pi_-^2 \expect_{X \sim p_-} \left[ \tilde{\ell}(f(X)) \right]
  .
  \label{eq:pairwise-similarity-conditional-sub2}
\end{align}
Combining Eqs.~\eqref{eq:risk-decomposition-sub2} and \eqref{eq:pairwise-similarity-conditional-sub2},
\begin{align}
  (\ast)
  &= \frac{\pi_+(1 + \pi_+)}{2} \expect_{X \sim p_+} \left[ \tilde{\ell}(f(X)) \right]
  - \frac{\pi_+(1 - \pi_+)}{2 \pi_-^2} \left\{ - \pi_+^2 \expect_{X \sim p_+} \left[ \tilde{\ell}(f(X)) \right] + \pi_\s \expect_{(X,X') \sim p_\s} \left[ \frac{\tilde{\ell}(f(X)) + \tilde{\ell}(f(X'))}{2} \right] \right\} \nonumber
  \\
  &= \frac{\pi_+}{2\pi_-} \expect_{X \sim p_+} \left[ \tilde{\ell}(f(X)) \right]
  - \frac{\pi_+\pi_\s}{2\pi_-} \expect_{(X,X') \sim p_\s} \left[ \frac{\tilde{\ell}(f(X)) + \tilde{\ell}(f(X'))}{2} \right]
  .
  \label{eq:risk-decomposition-sub3}
\end{align}
Finally from Eqs.~\eqref{eq:risk-decomposition-sub} and \eqref{eq:risk-decomposition-sub3}, the expected risk $R(f)$ is written as
\begin{align}
  R(f)
  &= \pi_\s \expect_{(X,X') \sim p_\s} \left[ \frac{\ell(f(X),-1) + \ell(f(X'),-1)}{2} \right]
  + \pi_\mathrm{D} \expect_{(X,X') \sim p_\mathrm{D}} \left[ \frac{\ell(f(X),-1) + \ell(f(X'),+1)}{2} \right] \nonumber \\
  & \quad + \frac{\pi_+}{2\pi_-} \expect_{X \sim p_+} \left[ \tilde{\ell}(f(X)) \right]
  - \frac{\pi_+\pi_\s}{2\pi_-} \expect_{(X,X') \sim p_\s} \left[ \frac{\tilde{\ell}(f(X)) + \tilde{\ell}(f(X'))}{2} \right] \nonumber
  \\
  &= R_{\mathrm{PSD},\ell}(f)
  .
\end{align}
\end{proof}

Now we give a proof for Theorem~\ref{thm:su-risk}.
\begin{proof}[Proof of Theorem~\ref{thm:su-risk}]
  By Lemma~\ref{lem:psd-estimator}, it is enough to show $R_{\mathrm{SU},\ell}(f) = R_{\mathrm{PSD},\ell}(f)$.

  From Eq.~\eqref{eq:pairwise-marginal},
  \begin{align}
    \expect_{X \sim p} & \left[ \frac{\ell(f(X),-1) + \ell(f(X),+1)}{2} \right] \nonumber
    \\
    &= \expect_{X \sim p}\left[\frac{\ell(f(X),-1)}{2}\right] + \expect_{X \sim p}\left[\frac{\ell(f(X)),+1}{2}\right]
    \qquad \qquad \text{($\because$ linearity of the expectation)}
    \nonumber
    \\
    &= \expect_{X \sim p}\left[\frac{\ell(f(X),-1)}{2}\right] + \expect_{X' \sim p}\left[\frac{\ell(f(X')),+1}{2}\right] \nonumber
    \\
    &= \expect_{X,X' \sim p} \left[ \frac{\ell(f(X),-1) + \ell(f(X'),+1)}{2} \right] \nonumber
    \\
    &= \pi_\s \expect_{(X,X') \sim p_\s} \left[ \frac{\ell(f(X),-1) + \ell(f(X'),+1)}{2} \right]
    + \pi_\mathrm{D} \expect_{(X,X') \sim p_\mathrm{D}} \left[ \frac{\ell(f(X),-1) + \ell(f(X'),+1)}{2} \right]
    \label{eq:pairwise-marginal-sub}
    ,
  \end{align}
  where $\mathbb{E}_{X \sim p}[\cdot]$ denotes the expectation over the marginal distribution $p(X)$
  and the last equality is obtained from Eq.~\eqref{eq:pairwise-marginal}.
  Eq.~\eqref{eq:pairwise-marginal-sub} produces an alternative expression of the expectation over $p_\d$ (the third term on the RHS of Eq.~\eqref{eq:psd-estimator}):
  \begin{align}
    \pi_\d & \expect_{(X,X') \sim p_\mathrm{D}} \left[ \frac{\ell(f(X),-1) + \ell(f(X'),+1)}{2} \right] \nonumber
    \\
    &= \expect_{X \sim p} \left[ \frac{\ell(f(X),-1) + \ell(f(X),+1)}{2} \right]
    - \pi_\s \expect_{(X,X') \sim p_\s} \left[ \frac{\ell(f(X),-1) + \ell(f(X'),+1)}{2} \right]
    \label{eq:expectation-of-dissimilar-alt}
    .
  \end{align}

  Next, we obtain an alternative expression of the expectation over positive data (the first term in RHS of Eq.~\eqref{eq:psd-estimator}).
  The following two equations~\eqref{eq:expectation-of-positive-alt-sub1} and \eqref{eq:expectation-of-positive-alt-sub2} are useful:
  \begin{align}
    \expect_{X \sim p} & \left[ \tilde{\ell}(f(X)) \right]
    = \pi_+ \expect_{X \sim p_+} \left[ \tilde{\ell}(f(X)) \right] + \pi_- \expect_{X \sim p_-} \left[ \tilde{\ell}(f(X)) \right] ,
    \label{eq:expectation-of-positive-alt-sub1}
  \end{align}
  which can simply be obtained from Eq.~\eqref{eq:marginal}.
  \begin{align}
    \pi_\s \expect_{(X,X') \sim p_\s} \left[ \frac{\tilde{\ell}(f(X)) + \tilde{\ell}(f(X'))}{2} \right]
    &= \pi_+^2 \expect_{X,X' \sim p_+} \left[\frac{\tilde{\ell}(f(X)) + \tilde{\ell}(f(X'))}{2}\right]
    + \pi_-^2 \expect_{X,X' \sim p_-} \left[\frac{\tilde{\ell}(f(X)) + \tilde{\ell}(f(X'))}{2}\right] \nonumber
    \\
    &= \pi_+^2 \left\{ \frac{1}{2}\expect_{X \sim p_+}\left[\tilde{\ell}(f(X))\right] + \frac{1}{2}\expect_{X' \sim p_+}\left[\tilde{\ell}(f(X'))\right] \right\} \nonumber \\
    & \qquad + \pi_-^2 \left\{ \frac{1}{2}\expect_{X \sim p_-}\left[\tilde{\ell}(f(X))\right] + \frac{1}{2}\expect_{X' \sim p_-}\left[\tilde{\ell}(f(X'))\right] \right\} \nonumber
    \\
    &= \pi_+^2 \expect_{X \sim p_+} \left[ \tilde{\ell}(f(X)) \right] + \pi_-^2 \expect_{X \sim p_-} \left[ \tilde{\ell}(f(X)) \right] ,
    \label{eq:expectation-of-positive-alt-sub2}
  \end{align}
  which is obtained from Eq.~\eqref{eq:pairwise-similarity-conditional}.
  By calculating $\text{\eqref{eq:expectation-of-positive-alt-sub2}}-\pi_-\times\text{\eqref{eq:expectation-of-positive-alt-sub1}}$ and organizing, we obtain
  \begin{align}
    \frac{\pi_+}{2\pi_-} & \expect_{X \sim p_+} \left[ \tilde{\ell}(f(X)) \right]
    = \frac{\pi_\s}{2\pi_-(2\pi_+ - 1)} \expect_{(X,X') \sim p_\s} \left[ \frac{\tilde{\ell}(f(X)) + \tilde{\ell}(f(X'))}{2} \right]
    - \frac{1}{2(2\pi_+ - 1)} \expect_{X \sim p} \left[ \tilde{\ell}(f(X)) \right]
    \label{eq:expectation-of-positive-alt}
    .
  \end{align}

  Substituting Eqs.~\eqref{eq:expectation-of-dissimilar-alt} and \eqref{eq:expectation-of-positive-alt} into Eq.~\eqref{eq:psd-estimator},
  \begin{align}
    R_{\mathrm{PSD},\ell}(f)
    &= \frac{\pi_\s}{2\pi_-(2\pi_+ - 1)} \expect_{(X,X') \sim p_\s} \left[ \frac{\tilde{\ell}(f(X)) + \tilde{\ell}(f(X'))}{2} \right]
    - \frac{1}{2(2\pi_+ - 1)} \expect_{X \sim p} \left[ \tilde{\ell}(f(X)) \right] \nonumber \\
    & \quad + \pi_\s \expect_{(X,X') \sim p_\s} \left[ -\frac{\pi_+}{2\pi_-}\frac{\ell(f(X),+1) + \ell(f(X'),+1)}{2} + \frac{1+\pi_-}{2\pi_-}\frac{\ell(f(X),-1) + \ell(f(X'), -1)}{2} \right] \nonumber \\
    & \quad + \expect_{X \sim p} \left[ \frac{\ell(f(X),-1) + \ell(f(X),+1)}{2} \right]
    - \pi_\s \expect_{(X,X') \sim p_\s} \left[ \frac{\ell(f(X),-1) + \ell(f(X'),+1)}{2} \right] \nonumber
    \\
    &= \pi_\s\expect_{(X,X') \sim p_\s} \left[ \frac{1 + 2\pi_+}{4(2\pi_+ - 1)}\tilde{\ell}(f(X)) + \frac{1 + 2\pi_-}{4(2\pi_+ - 1)}\tilde{\ell}(f(X')) \right] \nonumber \\
    & \quad + \expect_{X \sim p} \left[ -\frac{\pi_-}{2\pi_+ - 1}\ell(f(X), +1) + \frac{\pi_+}{2\pi_+ - 1}\ell(f(X), -1) \right] \nonumber
    \\
    &= \pi_\s\expect_{(X,X') \sim p_\s} \left[ \frac{\frac{1}{2\pi_+ - 1}\tilde{\ell}(f(X)) + \frac{1}{2\pi_+ - 1}\tilde{\ell}(f(X'))}{2} \right] \nonumber \\
    & \qquad + \expect_{X \sim p} \left[ -\frac{\pi_-}{2\pi_+ - 1}\ell(f(X), +1) + \frac{\pi_+}{2\pi_+ - 1}\ell(f(X), -1) \right]
    \label{eq:symmetric-condition}
    \\
    &= R_{\mathrm{SU},\ell}(f) \nonumber
    ,
  \end{align}
  which concludes the proof.
  The third equality of Eq.~\eqref{eq:symmetric-condition} holds because $X$ and $X'$ are symmetric and
  \begin{align*}
    \frac{1 + 2\pi_+}{4(2\pi_+ - 1)}\tilde{\ell}(\cdot) + \frac{1 + 2\pi_-}{4(2\pi_+ - 1)}\tilde{\ell}(\cdot)
    = \frac{1}{2\pi_+ - 1}\tilde{\ell}(\cdot)
    = \frac{\frac{1}{2\pi_+ - 1}\tilde{\ell}(\cdot) + \frac{1}{2\pi_+ - 1}\tilde{\ell}(\cdot)}{2}
    .
  \end{align*}
\end{proof}

\section{Discussion on Variance of Risk Estimator}

\subsection{Proof of Lemma~\ref{lem:pairing-fact}}
\label{sec:proof-of-pairing-fact}

The statement can be simply confirmed as follows:
\begin{align*}
  \expect_{(X,X') \sim p_\s} \left[ \frac{\mathcal{L}_{\s,\ell}(f(X)) + \mathcal{L}_{\s,\ell}(f(X'))}{2} \right]
  &= \expect_{X \sim \tilde{p}_\s} \left[ \frac{\mathcal{L}_{\s,\ell}(f(X))}{2} \right] + \expect_{X' \sim \tilde{p}_\s} \left[ \frac{\mathcal{L}_{\s,\ell}(f(X'))}{2} \right] \\
  &= \expect_{X \sim \tilde{p}_\s} \left[ \mathcal{L}_{\s,\ell}(f(X)) \right] \\
  &= \expect_{X \sim \tilde{p}_\s} \left[ \alpha\mathcal{L}_{\s,\ell}(f(X)) + (1-\alpha)\mathcal{L}_{\s,\ell}(f(X)) \right] \\
  &= \expect_{X \sim \tilde{p}_\s} \left[ \alpha\mathcal{L}_{\s,\ell}(f(X)) \right] + \expect_{X' \sim \tilde{p}_\s} \left[ (1-\alpha)\mathcal{L}_{\s,\ell}(f(X')) \right] \\
  &= \expect_{(X,X') \sim p_\s} \left[ \alpha\mathcal{L}_{\s,\ell}(f(X)) + (1-\alpha)\mathcal{L}_{\s,\ell}(f(X')) \right]
  .
\end{align*}
\qed

\subsection{Proof of Theorem~\ref{thm:minimum-variance-pairing}}
\label{sec:proof-of-minimum-variance}

We show Eq.~\eqref{eq:minimum-variance-s-estimator} is the variance minimizer of
\begin{align*}
  S(f;\alpha) \triangleq
  \frac{1}{n_\s} \sum_{i=1}^{n_\s} \left\{ \alpha\mathcal{L}_{\s,\ell}(f(\bm{x}_{\s,i})) + (1-\alpha)\mathcal{L}_{\s,\ell}(f(\bm{x}_{\s,i}')) \right\}
  ,
\end{align*}
with respect to $\alpha \in \mathbb{R}$.
Let $\mu_1 \triangleq \expect_{\{(\bm{x}_{\s,i},\bm{x}_{\s,i}')\} \sim p_\s}[S(f;\alpha)]$ and
\begin{align*}
  \tilde{\mu}_1 &\triangleq \expect_{\{\bm{x}_{\s,i}\} \sim \tilde{p}_\s} \left[ \frac{1}{n_\s}\sum_{i=1}^{n_\s}\mathcal{L}_{\s,\ell}(f(\bm{x}_{\s,i})) \right]
  = \expect_{\{\bm{x}_{\s,i}'\} \sim \tilde{p}_\s} \left[ \frac{1}{n_\s}\sum_{i=1}^{n_\s}\mathcal{L}_{\s,\ell}(f(\bm{x}_{\s,i}')) \right],
  \\
  \tilde{\mu}_2 &\triangleq \expect_{\{\bm{x}_{\s,i}\} \sim \tilde{p}_\s} \left[ \left( \frac{1}{n_\s}\sum_{i=1}^{n_\s}\mathcal{L}_{\s,\ell}(f(\bm{x}_{\s,i})) \right)^2 \right]
  = \expect_{\{\bm{x}_{\s,i}'\} \sim \tilde{p}_\s} \left[ \left( \frac{1}{n_\s}\sum_{i=1}^{n_\s}\mathcal{L}_{\s,\ell}(f(\bm{x}_{\s,i}')) \right)^2 \right]
  .
\end{align*}
Then,
\begin{align*}
  \mathrm{Var}(S(f;\alpha))
  &= \expect_{\{(\bm{x}_{\s,i},\bm{x}_{\s,i}')\} \sim p_\s} \left[ \left( S(f;\alpha) - \mu_1 \right)^2 \right]
  \\
  &= \expect_{\{(\bm{x}_{\s,i},\bm{x}_{\s,i}')\} \sim p_\s} \left[ S(f;\alpha)^2 \right] - \mu_1^2
  \\
  &= \alpha^2 \expect_{\{\bm{x}_{\s,i}\} \sim \tilde{p}_\s} \left[ \left( \frac{1}{n_\s}\sum_{i=1}^{n_\s}\mathcal{L}_{\s,\ell}(f(\bm{x}_{\s,i})) \right)^2 \right] \\
  & \qquad + 2\alpha(1-\alpha)
    \expect_{\{\bm{x}_{\s,i}\} \sim \tilde{p}_\s} \left[ \frac{1}{n_\s}\sum_{i=1}^{n_\s}\mathcal{L}_{\s,\ell}(f(\bm{x}_{\s,i})) \right]
    \expect_{\{\bm{x}_{\s,i}'\} \sim \tilde{p}_\s} \left[ \frac{1}{n_\s}\sum_{i=1}^{n_\s}\mathcal{L}_{\s,\ell}(f(\bm{x}_{\s,i}')) \right] \\
  & \qquad + (1-\alpha)^2 \expect_{\{\bm{x}_{\s,i}'\} \sim \tilde{p}_\s} \left[ \left( \frac{1}{n_\s}\sum_{i=1}^{n_\s}\mathcal{L}_{\s,\ell}(f(\bm{x}_{\s,i}')) \right)^2 \right]
  - \mu_1^2
  \\
  &= \tilde{\mu}_2\alpha^2 + 2\tilde{\mu}_1^2\alpha(1-\alpha) + \tilde{\mu}_2(1-\alpha)^2 - \mu_1^2
  \\
  &= 2(\tilde{\mu}_2 - \tilde{\mu}_1^2)\left\{ \left( \alpha - \frac{1}{2} \right)^2 - \frac{1}{4} \right\} + \tilde{\mu}_2 - \mu_1^2
  .
\end{align*}
Since $\tilde{\mu}_2 - \tilde{\mu}_1^2$ is the variance of $\frac{1}{n_\s}\sum_i\mathcal{L}_{\s,\ell}(f(\bm{x}_{\s,i}))$, $\tilde{\mu}_2 - \tilde{\mu}_1^2 \ge 0$.
Thus, $\mathrm{Var}(S(f;\alpha))$ is minimized when $\alpha = \frac{1}{2}$.
\qed

\section{Proof of Theorem~\ref{thm:convex-su}}
\label{sec:proof-of-convex-su}

Since $\ell$ is a twice differentiable margin loss,
there is a twice differentiable function $\psi: \mathbb{R} \rightarrow \mathbb{R}_+$ such that $\ell(z,t) = \psi(tz)$.
Taking the derivative of
\begin{align*}
\hat{J}_\ell(\bm{w})
&= \frac{\lambda}{2}\bm{w}^\top\bm{w}
-\frac{\pi_\s}{2n_\s(2\pi_+ - 1)} \sum_{i=1}^{2n_\s} \bm{w}^\top\bm{\phi}(\tilde{\bm{x}}_{\s,i}) \\
& \qquad + \frac{1}{n_\u(2\pi_+ - 1)} \sum_{i=1}^{n_\u} \left\{ -\pi_-\ell(\bm{w}^\top\bm{\phi}(\bm{x}_{\u,i}),+1) + \pi_+\ell(\bm{w}^\top\bm{\phi}(\bm{x}_{\u,i}),-1) \right\}
\end{align*}
with respect to $\bm{w}$,
\begin{align*}
\frac{\partial}{\partial\bm{w}}\hat{J}_\ell(\bm{w})
&= \lambda\bm{w}
- \frac{\pi_\s}{2n_\s(2\pi_+ - 1)} \sum_{i=1}^{2n_\s} \bm{\phi}(\tilde{\bm{x}}_{\s,i})
+ \frac{1}{n_\u(2\pi_+ - 1)} \sum_{i=1}^{n_\u} \left\{ -\pi_-\frac{\partial\ell(\xi_i,+1)}{\partial\xi_i} + \pi_+\frac{\partial\ell(\xi_i,-1)}{\partial\xi_i} \right\}\bm{\phi}(\bm{x}_{\u,i})
,
\end{align*}
where $\xi_i \triangleq \bm{w}^\top\bm{\phi}(\bm{x}_{\u,i})$.
Here, the second-order derivative of $\ell(z,t)$ with respect to $z$ is
\begin{align*}
\frac{\partial^2\ell(z,t)}{\partial z^2}
= \frac{\partial^2\psi(tz)}{\partial z^2}
= \frac{\partial}{\partial z}\left(t\frac{\partial\psi(\xi)}{\partial \xi}\right)
= t^2\frac{\partial^2\psi(\xi)}{\partial\xi^2}
= \frac{\partial^2\psi(\xi)}{\partial\xi^2}
,
\end{align*}
where $\xi = tz$ is employed in the second equality and $t \in \{+1,-1\}$ is employed in the last equality.
Thus the Hessian of $\hat{J}_\ell$ is
\begin{align*}
\bm{H}\hat{J}_\ell(\bm{w})
&= \lambda I + \frac{1}{n_\u(2\pi_+ - 1)} \sum_{i=1}^{n_\u}
\left\{ -\pi_-\frac{\partial}{\partial\bm{w}}\frac{\partial\ell(\xi_i,+1)}{\partial\xi_i} +\pi_+\frac{\partial}{\partial\bm{w}}\frac{\partial\ell(\xi_i,-1)}{\partial\xi_i} \right\}\bm{\phi}(\bm{x}_{\u,i})^\top
\\
&= \lambda I + \frac{1}{n_\u(2\pi_+ - 1)} \sum_{i=1}^{n_\u}
\left\{ -\pi_-\frac{\partial^2\ell(\xi_i,+1)}{\partial\xi_i^2}\frac{\partial\xi_i}{\partial\bm{w}} + \pi_+\frac{\partial^2\ell(\xi_i,-1)}{\partial\xi_i^2}\frac{\partial\xi_i}{\partial\bm{w}} \right\}\bm{\phi}(\bm{x}_{\u,i})^\top
\\
&= \lambda I + \frac{1}{n_\u(2\pi_+ - 1)} \sum_{i=1}^{n_\u}
\left\{-\pi_-\frac{\partial^2\ell(\xi_i,+1)}{\partial\xi_i^2} + \pi_+\frac{\partial^2\ell(\xi_i,-1)}{\partial\xi_i^2} \right\}\bm{\phi}(\bm{x}_{\u,i})\bm{\phi}(\bm{x}_{\u,i})^\top
\\
&= \lambda I + \frac{1}{n_\u(2\pi_+ - 1)} \sum_{i=1}^{n_\u}
(\pi_+ - \pi_-)\frac{\partial^2\psi(\xi)}{\partial \xi^2} \bm{\phi}(\bm{x}_{\u,i})\bm{\phi}(\bm{x}_{\u,i})^\top
\\
&= \lambda I + \frac{1}{n_\u} \frac{\partial^2\psi(\xi)}{\partial \xi^2} \sum_{i=1}^{n_\u} \bm{\phi}(\bm{x}_{\u,i})\bm{\phi}(\bm{x}_{\u,i})^\top
\\
& \succeq 0
,
\end{align*}
where $A \succeq 0$ means that a matrix $A$ is positive semidefinite.
Positive semidefiniteness of $\bm{H}\hat{J}_\ell(\bm{w})$ follows from $\frac{\partial^2\psi(\xi)}{\partial \xi^2} \ge 0$ ($\because$ $\ell$ is convex) and $\bm{\phi}(\bm{x}_{\u,i})\bm{\phi}(\bm{x}_{\u,i})^\top \succeq 0$.
Thus $\hat{J}_\ell(\bm{w})$ is convex.
\qed

\section{Derivation of Optimization Problems}
\label{sec:optimization-problem}

\subsection{Squared Loss}

First, substituting the linear-in-parameter model $f(\bm{x}) = \bm{w}^\top\bm{\phi}(\bm{x})$ and the squared loss $\ell_\mathrm{SQ}(z,t) = \frac{1}{4}(tz - 1)^2$ into Eq.~\eqref{eq:su-objective-function},
we obtain the following objective function:
\begin{align*}
  \hat{J}_\mathrm{SQ}(\bm{w})
  &= \frac{\pi_\s}{2(2\pi_+ - 1)n_\s} \sum_{i=1}^{2n_\s} \frac{(\bm{w}^\top\bm{\phi}(\tilde{\bm{x}}_{\s,i}) - 1)^2 - (\bm{w}^\top\bm{\phi}(\tilde{\bm{x}}_{\s,i}) + 1)^2}{4} \\
  & \quad + \frac{1}{n_\u} \sum_{i=1}^{n_\u} \frac{-\pi_-\cdot\tfrac{1}{4}(\bm{w}^\top\bm{\phi}(\bm{x}_{\u,i}) - 1)^2 + \pi_+\cdot\tfrac{1}{4}(\bm{w}^\top\bm{\phi}(\bm{x}_{\u,i}) + 1)^2}{2\pi_+ - 1} \\
  & \quad + \frac{\lambda}{2}\|\bm{w}\|^2
  \\
  &= \frac{1}{2\pi_+ - 1} \left\{ - \frac{\pi_\s}{2n_\s} \sum_{i=1}^{2n_\s} \bm{w}^\top\bm{\phi}(\tilde{\bm{x}}_{\s,i})
  + \frac{1}{4n_\u} \sum_{i=1}^{n_\u} \left\{ (2\pi_+ - 1)(\bm{w}^\top\bm{\phi}(\bm{x}_{\u,i})\bm{\phi}(\bm{x}_{\u,i})^\top\bm{w} + 1) + 2\bm{\phi}(\bm{x}_{\u,i})^\top\bm{w} \right\} \right\} \\
  & \quad + \frac{\lambda}{2}\|\bm{w}\|^2
  \\
  &= \bm{w}^\top\left(\frac{1}{4n_\u}X_\u^\top X_\u + \frac{\lambda}{2}I\right)\bm{w}
  + \frac{1}{2\pi_+ - 1}\left(-\frac{\pi_\s}{2n_\s}\bm{1}^\top X_\s + \frac{1}{2n_\u}\bm{1}^\top X_\u\right)\bm{w}
  .
\end{align*}
Taking the derivative with respect to $\bm{w}$,
\begin{align*}
  \frac{\partial}{\partial\bm{w}}\hat{J}_\mathrm{SQ}(\bm{w})
  &= \frac{1}{2n_\u} \left(X_\u^\top X_\u + 2n_\u\lambda I\right)\bm{w}
  - \frac{1}{2\pi_+ - 1}\left(\frac{\pi_\s}{2n_\s}X_\s^\top\bm{1} - \frac{1}{2n_\u}X_\u^\top\bm{1}\right)
  .
\end{align*}
Solving $\frac{\partial}{\partial\bm{w}}\hat{J}_\mathrm{SQ}(\bm{w}) = 0$, we obtain the analytical solution:
\begin{align*}
  \bm{w} = \frac{n_\u}{2\pi_+ - 1}(X_\u^\top X_\u + 2n_\u\lambda I)^{-1}\left(\frac{\pi_\s}{n_\s}X_\s^\top\bm{1} - \frac{1}{n_\u}X_\u^\top\bm{1}\right)
  .
\end{align*}

\subsection{Double-Hinge Loss}

Using the double-hinge loss $\ell_\mathrm{DH}(z,t)=\max(-tz,\max(0,\frac{1}{2}-\frac{1}{2}tz))$,
we obtain the following objective function:
\begin{align*}
  \hat{J}_\mathrm{DH}(\bm{w})
  &= -\frac{\pi_\s}{2n_\s(2\pi_+ - 1)} \sum_{i=1}^{2n_\s} \bm{w}^\top\bm{\phi}(\tilde{\bm{x}}_{\s,i}) \\
  & \qquad + \frac{1}{n_\u(2\pi_+ - 1)} \sum_{i=1}^{n_\u} \left\{ -\pi_-\ell_{\mathrm{DH}}(\bm{w}^\top\bm{\phi}(\bm{x}_{\u,i})) + \pi_+\ell_{\mathrm{DH}}(-\bm{w}^\top\bm{\phi}(\bm{x}_{\u,i})) \right\} \\
  & \qquad + \frac{\lambda}{2}\bm{w}^\top\bm{w}
  .
\end{align*}
Using slack variables $\bm{\xi}, \bm{\eta} \in \mathbb{R}^{n_\u}$, the objective function can be rewritten into the following optimization problem:
\begin{align*}
  \min_{\bm{w},\bm{\xi},\bm{\eta}} &
  -\frac{\pi_\s}{2n_\s(2\pi_+ - 1)} \bm{1}^\top X_\s\bm{w}
  -\frac{\pi_-}{n_\s(2\pi_+ - 1)} \bm{1}^\top \bm{\xi}
  +\frac{\pi_+}{n_\u(2\pi_+ - 1)} \bm{1}^\top \bm{\eta}
  +\frac{\lambda}{2} \bm{w}^\top\bm{w}
\end{align*}
\begin{alignat*}{3}
  \text{s.t.} \qquad
  & \bm{\xi} \ge \bm{0}, \quad
  & \bm{\xi} \ge \frac{1}{2}\bm{1} + \frac{1}{2}X_\u\bm{w}, \quad
  & \bm{\xi} \ge X_\u\bm{w}, \\
  & \bm{\eta} \ge \bm{0}, \quad
  & \bm{\eta} \ge \frac{1}{2}\bm{1} - \frac{1}{2}X_\u\bm{w}, \quad
  & \bm{\eta} \ge -X_\u\bm{w},
\end{alignat*}
where $\ge$ for vectors denotes the element-wise inequality.

Below, we rewrite the optimization problem into the standard QP form.
Let $\bm{\gamma} \triangleq [\bm{w}^\top\; \bm{\xi}^\top\; \bm{\eta}^\top]^\top \in \mathbb{R}^{d + 2n_\u}$ be a objective variable and
\begin{alignat*}{3}
  P &\triangleq \left[
  \begin{array}{ccc}
    \lambda I_{d} & O_{d,n_\u} & O_{d,n_\u} \\
    O_{n_\u,d} & O_{n_\u,n_\u} & O_{n_\u,n_\u} \\
    O_{n_\u,d} & O_{n_\u,n_\u} & O_{n_\u,n_\u}
  \end{array}
  \right], \qquad &
  \bm{q} &\triangleq \left[
  \begin{array}{c}
    -\frac{\pi_\s}{2n_\s(2\pi_+ - 1)}X_\s^\top\bm{1}_{d} \\
    -\frac{\pi_-}{n_\u(2\pi_+ - 1)}\bm{1}_{n_\u} \\
    \frac{\pi_+}{n_\u(2\pi_+ - 1)}\bm{1}_{n_\u}
  \end{array}
  \right]
  \\
  G &\triangleq \left[
  \begin{array}{ccc}
    O_{n_\u,d} & -I_{n_\u} & O_{n_\u,n_\u} \\
    \frac{1}{2}X_\u & -I_{n_\u} & O_{n_\u,n_\u} \\
    X_\u & -I_{n_\u} & O_{n_\u,n_\u} \\
    O_{n_\u,d} & O_{n_\u,n_\u} & -I_{n_\u}\\
    -\frac{1}{2}X_\u & O_{n_\u,n_\u} & -I_{n_\u} \\
    -X_\u & O_{n_\u,n_\u} & -I_{n_\u}
  \end{array}
  \right], \qquad &
  \bm{h} &\triangleq \left[
  \begin{array}{c}
    \bm{0}_{n_\u} \\
    -\frac{1}{2}\bm{1}_{n_\u} \\
    \bm{0}_{n_\u} \\
    \bm{0}_{n_\u} \\
    -\frac{1}{2}\bm{1}_{n_\u} \\
    \bm{0}_{n_\u}
  \end{array}
  \right]
  ,
\end{alignat*}
where $I_k$ means $k \times k$ identity matrix,
$O_{k,l}$ means $k \times l$ all-zero matrix,
$\bm{1}_k$ is $k$-dimensional all-one vector,
and $\bm{0}_k$ is $k$-dimensional all-zero vector.
Then the optimization problem is
\begin{align*}
  \min_{\bm{\gamma}} \frac{1}{2}\bm{\gamma}^\top P \bm{\gamma} + \bm{q}^\top\bm{\gamma} \qquad \text{s.t.} \quad G\bm{\gamma} \le \bm{h}
  ,
\end{align*}
which is the standard form of QP.

\section{Proof of Theorem~\ref{thm:estimation-error-bound}}
\label{sec:proof-of-estimation-error-bound}

\newcommand{\tildesu}{\tilde{\s}\u}

First, we derive the next risk expression for convenience.
\begin{lemma}
  \label{lem:su-estimator}
  Given any function $f: \mathcal{X} \rightarrow \mathbb{R}$, let $R_{\tildesu,\ell}(f)$ be
  \begin{align*}
    R_{\tildesu,\ell}(f) = \pi_\s \expect_{X \sim \tilde{p}_{\s}} \left[ \mathcal{L}_{\s,\ell}(f(X)) \right] + \expect_{X \sim p} \left[ \mathcal{L}_{\u,\ell}(f(X)) \right]
    ,
  \end{align*}
  then $R_{\mathrm{SU},\ell}(f) = R_{\tildesu,\ell}(f)$.
\end{lemma}

\begin{proof}
  The first term on the RHS of Eq.~\eqref{eq:su-estimator} can be transformed as follows:
  \begin{align*}
    \pi_\s \expect_{(X,X') \sim p_\s} \left[ \frac{\mathcal{L}_{\s,\ell}(f(X)) + \mathcal{L}_{\s,\ell}(f(X'))}{2} \right]
    &= \pi_\s \left\{ \frac{1}{2} \expect_{X \sim \tilde{p}_\s} \left[ \mathcal{L}_{\s,\ell}(f(X)) \right]
    + \frac{1}{2} \expect_{X' \sim \tilde{p}_\s} \left[ \mathcal{L}_{\s,\ell}(f(X')) \right] \right\}
    \\
    &= \pi_\s \expect_{X \sim \tilde{p}_\s} \left[ \mathcal{L}_{\s,\ell}(f(X)) \right]
    .
  \end{align*}
\end{proof}

Next, we show the uniform deviation bound, which is useful to derive estimation error bounds.
The proof can be found in the textbooks such as \citet{Mohri:2012} (Theorem 3.1).
\begin{lemma}
  \label{lem:uniform-deviation-bound}
  Let $Z$ be a random variable drawn from a probability distribution with density $\mu$,
  $\mathcal{H} = \{h: \mathcal{Z} \rightarrow [0, M]\}$ ($M > 0$) be a class of measurable functions,
  $\{z_i\}_{i=1}^n$ be i.i.d.~samples drawn from the distribution with density $\mu$.
  Then, with probability at least $1 - \delta$,
  \begin{align*}
    \sup_{h \in \mathcal{H}} \left| \expect_{Z \sim \mu}[h(Z)] - \frac{1}{n}\sum_{i=1}^n h(z_i) \right| \le & 2\mathfrak{R}(\mathcal{H};n,\mu) + \sqrt{\frac{M^2\log\frac{2}{\delta}}{2n}}
    .
  \end{align*}
\end{lemma}

Let us begin with the estimation error $R(\hat{f}) - R(f^*)$.
For convenience, let
\begin{alignat*}{2}
  R_{\tilde{\s},\ell}(f) &\triangleq \expect_{X \sim \tilde{p}_\s} \left[ \mathcal{L}_{\s,\ell}(f(X)) \right], \qquad
  \hat{R}_{\tilde{\s},\ell}(f) &\triangleq \frac{1}{2n_\s}\sum_{i=1}^{2n_\s} \mathcal{L}_{\s,\ell}(f(\tilde{\bm{x}}_{\s,i})), \\
  R_{\u,\ell}(f) &\triangleq \expect_{X \sim p} \left[ \mathcal{L}_{\u,\ell}(f(X)) \right], \qquad
  \hat{R}_{\u,\ell}(f) &\triangleq \frac{1}{n_\u}\sum_{i=1}^{n_\u} \mathcal{L}_{\u,\ell}(f(\bm{x}_{\u,i})), \\
  \hat{R}_{\tildesu,\ell}(f) &\triangleq \pi_\s \hat{R}_{\tilde{\s},\ell}(f) + \hat{R}_{\u,\ell}(f).
\end{alignat*}
Note that
\begin{align}
  \hat{R}_{\s\u,\ell}(f) = \hat{R}_{\tilde{\s}\u,\ell}(f)
  \label{eq:su-risk-equivalence}
\end{align}
by Eq.~\eqref{eq:su-estimator}.
Then,
\begin{align}
  R(\hat{f}) - R(f^*)
  &= R_{\s\u,\ell}(\hat{f}) - R_{\s\u,\ell}(f^*)
  && \text{($\because$ Theorem~\ref{thm:su-risk})} \nonumber
  \\
  &= (R_{\s\u,\ell}(\hat{f}) - \hat{R}_{\s\u,\ell}(\hat{f})) + (\hat{R}_{\s\u,\ell}(\hat{f}) - \hat{R}_{\s\u,\ell}(f^*)) \nonumber \\
  & \qquad + (\hat{R}_{\s\u,\ell}(f^*) - R_{\s\u,\ell}(f^*))
  \nonumber
  \\
  & \le (R_{\s\u,\ell}(\hat{f}) - \hat{R}_{\s\u,\ell}(\hat{f})) + 0 + (\hat{R}_{\s\u,\ell}(f^*) - R_{\s\u,\ell}(f^*))
  && \text{($\because$ by the definition of $f^*$ and $\hat{f}$)} \nonumber
  \\
  & \le 2\sup_{f \in \mathcal{F}} \left| R_{\s\u,\ell}(f) - \hat{R}_{\s\u,\ell}(f) \right| \nonumber
  \\
  & = 2\sup_{f \in \mathcal{F}} \left| R_{\tildesu,\ell}(f) - \hat{R}_{\tildesu,\ell}(f) \right|
  && \text{($\because$ Lemma~\ref{lem:su-estimator} and Eq.~\eqref{eq:su-risk-equivalence})} \nonumber
  \\
  & \le 2\pi_\s \sup_{f \in \mathcal{F}} \left| R_{\tilde{\s},\ell}(f) - \hat{R}_{\tilde{\s},\ell}(f) \right| + 2\sup_{f \in \mathcal{F}} \left| R_{\u,\ell}(f) - \hat{R}_{\u,\ell}(f) \right|
  && \text{($\because$ subaditivity of $\sup$)}
  \label{eq:su-risk-decomposition}
  .
\end{align}
Each term in the last line is bounded in next two lemmas with probability at least $1 - \frac{\delta}{2}$.

\begin{lemma}
  \label{lem:rs-bound}
  Assume the loss function $\ell$ is $\rho$-Lipschitz with respect to the first argument ($0 < \rho < \infty$),
  and all functions in the model class $\mathcal{F}$ are bounded, i.e.,
  there exists a constant $C_b$ such that $\|f\|_\infty \le C_b$ for any $f \in \mathcal{F}$.
  Let $C_\ell \triangleq \sup_{t\in\{\pm 1\}} \ell(C_b,t)$.
  For any $\delta > 0$, with probability at least $1 - \frac{\delta}{2}$,
  \begin{align*}
    \sup_{f \in \mathcal{F}} \left| R_{\tilde{\s},\ell}(f) - \hat{R}_{\tilde{\s},\ell}(f) \right|
    \le \frac{4\rho C_\mathcal{F} + \sqrt{2C_\ell^2\log\frac{4}{\delta}}}{|2\pi_+ - 1|\sqrt{2n_\s}}
  \end{align*}
\end{lemma}

\begin{proof}
  By Lemma~\ref{lem:uniform-deviation-bound},
  \begin{align*}
    \sup_{f \in \mathcal{F}} \left| R_{\tilde{\s},\ell}(f) - \hat{R}_{\tilde{\s},\ell}(f) \right|
    &= \frac{1}{|2\pi_+ - 1|} \sup_{f \in \mathcal{F}} \left| \expect_{X \sim \tilde{p}_\s}\left[ \tilde{\ell}(f(X)) \right] - \frac{1}{2n_\s} \sum_{i=1}^{2n_\s} \tilde{\ell}(f(\tilde{\bm{x}}_{\s,i})) \right|
    \\
    &\le \frac{1}{|2\pi_+ - 1|} \left\{ \sup_{f \in \mathcal{F}} \left| \expect_{X \sim \tilde{p}_\s}\left[ \ell(f(X),+1) \right] - \frac{1}{2n_\s} \sum_{i=1}^{2n_\s} \ell(f(\tilde{\bm{x}}_{\s,i}),+1) \right| \right. \\
    & \quad \left. + \sup_{f \in \mathcal{F}} \left| \expect_{X \sim \tilde{p}_\s}\left[ \ell(f(X),-1) \right] - \frac{1}{2n_\s} \sum_{i=1}^{2n_\s} \ell(f(\tilde{\bm{x}}_{\s,i}),-1) \right| \right\}
    \\
    &\le \frac{1}{|2\pi_+ - 1|}\left\{ 4\mathfrak{R}(\ell \circ \mathcal{F}; 2n_\s,p_\s) + \sqrt{\frac{2C_\ell^2\log\frac{4}{\delta}}{2n_\s}} \right\}
    ,
  \end{align*}
  where $\ell \circ \mathcal{F}$ in the last line means $\{\ell \circ f \mid f \in \mathcal{F}\}$.
  The last inequality holds from Lemma~\ref{lem:uniform-deviation-bound}.
  By Talagrand's lemma (Lemma 4.2 in \citet{Mohri:2012}),
  \begin{align*}
    \mathfrak{R}(\ell\circ\mathcal{F};2n_\s,p_\s) \le \rho\mathfrak{R}(\mathcal{F};2n_\s,p_\s)
    .
  \end{align*}
  Together with Eq.~\eqref{eq:model-assumption}, we obtain
  \begin{align*}
    \sup_{f \in \mathcal{F}} \left| R_{\tilde{\s},\ell}(f) - \hat{R}_{\tilde{\s},\ell}(f) \right|
    &\le \frac{1}{|2\pi_+ - 1|} \left\{ 4\rho \frac{C_\mathcal{F}}{\sqrt{2n_\s}} + \sqrt{\frac{2C_\ell^2\log\frac{4}{\delta}}{2n_\s}} \right\}
    \\
    &= \frac{4\rho C_\mathcal{F} + \sqrt{2C_\ell^2\log\frac{4}{\delta}}}{|2\pi_+ - 1|\sqrt{2n_\s}}
    .
  \end{align*}
\end{proof}

\begin{lemma}
  \label{lem:ru-bound}
  Assume the loss function $\ell$ is $\rho$-Lipschitz with respect to the first argument ($0 < \rho < \infty$),
  and all functions in the model class $\mathcal{F}$ are bounded, i.e.,
  there exists a constant $C_b$ such that $\|f\|_\infty \le C_b$ for any $f \in \mathcal{F}$.
  Let $C_\ell \triangleq \sup_{t\in\{\pm 1\}} \ell(C_b,t)$.
  For any $\delta > 0$, with probability at least $1 - \frac{\delta}{2}$,
  \begin{align*}
    \sup_{f \in \mathcal{F}} \left| R_{\u,\ell}(f) - \hat{R}_{\u,\ell}(f) \right|
    \le \frac{2\rho C_\mathcal{F} + \sqrt{\frac{1}{2}C_\ell^2\log\frac{4}{\delta}}}{|2\pi_+ - 1|\sqrt{n_\u}}
  \end{align*}
\end{lemma}

\begin{proof}
  This lemma can be proven similarly to Lemma~\ref{lem:rs-bound}.
\end{proof}

Combining Lemma~\ref{lem:rs-bound}, Lemma~\ref{lem:ru-bound} and Eq.~\eqref{eq:su-risk-decomposition}, Theorem~\ref{thm:estimation-error-bound} is proven.
\qed


\end{document}